\documentclass[letterpaper]{article} %
\usepackage{aaai23}  %
\usepackage{times}  %
\usepackage{helvet}  %
\usepackage{courier}  %
\usepackage[hyphens]{url}  %
\usepackage{graphicx} %
\usepackage{natbib}  %
\usepackage{caption} %
\usepackage{amsmath}
\usepackage{amssymb}
\usepackage{amsthm}
\usepackage{hyperref}       %

\usepackage{booktabs}       %
\usepackage{amsfonts}       %
\usepackage{nicefrac}       %
\usepackage{microtype}      %
\usepackage{xcolor}         %

\usepackage{bm}
\usepackage{algcompatible}
\usepackage[ruled,vlined]{algorithm2e}

\usepackage{graphicx}
\usepackage{url}

\usepackage{hyperref}
\hypersetup{
    colorlinks=true,
    linkcolor={blue!50!black},
    citecolor={blue!50!black},
    urlcolor={blue!80!black}
}

\usepackage{mdframed}
\usepackage{subfigure}
\usepackage{enumitem}
\usepackage{amsthm}
\usepackage{thmtools,thm-restate}
\usepackage{wrapfig}
\usepackage{upgreek}
\usepackage{hyperref}

\theoremstyle{definition}
\newtheorem{definition}{Definition}[section]

\newtheorem{proposition}{Proposition}

\graphicspath{ {figures/} }

\usepackage{floatrow}

\usepackage{amsmath,amsfonts,bm}

\newcommand{\brck}[1]{\left(#1\right)}

\newcommand{\brcksq}[1]{\left[#1\right]}

\def\eqref#1{equation~\ref{#1}}

\def\1{\bm{1}}

\DeclareMathAlphabet{\mathsfit}{\encodingdefault}{\sfdefault}{m}{sl}
\SetMathAlphabet{\mathsfit}{bold}{\encodingdefault}{\sfdefault}{bx}{n}

\def\gH{{\mathcal{H}}}

\DeclareMathOperator*{\argmax}{arg\,max}
\DeclareMathOperator*{\argmin}{arg\,min}

\def\principal{principal}
\def\principals{principals}

\def\subagents{agents}
\mathchardef\mhyphen="2D

\newcommand{\ouralgoshort}{Ours}
\newcommand{\ouralgolong}{our algorithm}
\newcommand{\agentplayer}{agent}
\newcommand{\agentplayers}{agents}
\def\robustagent{principal}

\def\ourpolicy{policy}

\def\otheragents{other agents}

\newcommand{\ouralgo}{our algorithm}

\def\bimatrixactor{gambler}
\def\robustsubscript{0}

\def\AdversarialAgent{Adv}
\def\RiskAverseAgent{RiskAv}

\def\numagents{N}
\def\numplayers{\numagents}

\def\uncertaintyset{\bm{X}}
\def\behavior{\uncertaintyset}

\def\state{s}
\def\action{a}
\def\vaction{\bm{a}}
\def\actionset{A}
\def\actionsetoneton{A_{1:n}}
\def\bfaction{\vaction}

\def\equilibrium{\textrm{EQ}}

\newcommand\pdfover[1]{P(#1)}
\newcommand\prodpdfover[1]{P^\textrm{prod}(#1)}

\def\argmax{\text{argmax}}
\def\game{G}
\def\exputility{\overline{\utility}}
\def\reals{\mathbb{R}}

\def\policy{x}
\def\vpolicy{\bm{x}}
\def\bfstrat{\vpolicy}

\def\MNE{\textrm{MNE}}
\def\CCE{\textrm{CCE}}
\newcommand{\norm}[1]{\lVert #1 \rVert}
\def\specutility{\nu}
\def\decproblem{\Upgamma}

\def\EE{\mathbb{E}}

\def\fixedlambda{\lambda_0}

\def\robustagentreward{\rew_\robustsubscript}

\newcommand{\mbf}[1]{{{\color{blue} $\star$ \textbf{#1}}}}
\newcommand{\mmbf}[1]{{\textbf{#1}}}

\def\rew{r}

\def\util{u}

\def\utility{\util}

\def\endow{x}

\newcommand{\posttax}[1]{\tilde{#1}}

\def\income{z}

\def\labor{l}

\def\isoeta{\eta}
\def\tax{T}

\def\gini{\textit{gini}}

\def\socialwelfare{\textit{swf}}
\def\SWF{\socialwelfare}

\title{Learning to Play General-Sum Games Against Multiple Boundedly Rational Agents}

\author{
Eric Zhao\textsuperscript{\rm 1,2}, Alexander R. Trott\textsuperscript{\rm 1}, Caiming Xiong\textsuperscript{\rm 1}, Stephan Zheng\textsuperscript{\rm 1}
}
\affiliations{
    \textsuperscript{\rm 1}Salesforce Research. Palo Alto, California, USA\\
    \textsuperscript{\rm 2}University of California, Berkeley. Berkeley, California, USA
}

\begin{document}

\maketitle

\begin{abstract}
We study the problem of training a \principal{} in a multi-agent general-sum game using reinforcement learning (RL).
Learning a robust \principal{} policy requires anticipating the worst possible strategic responses of other agents, which is generally NP-hard.
However, we show that no-regret dynamics can identify these worst-case responses in poly-time in smooth games.
We propose a framework that uses this policy evaluation method for efficiently learning a robust \principal{} policy using RL.
This framework can be extended to provide robustness to boundedly rational agents too.
Our motivating application is automated mechanism design: we empirically demonstrate our framework learns robust mechanisms in both matrix games and complex spatiotemporal games\footnote{Source code for these experiments are released at \url{https://github.com/salesforce/strategically-robust-ai}.}.
In particular, we learn a dynamic tax policy that improves the welfare of a simulated trade-and-barter economy by 15\%, even when facing previously unseen boundedly rational RL taxpayers.
\end{abstract}

\section{Introduction}
We study the problem of learning a \principal{} policy in a general-sum game against boundedly rational agents.
Learning a robust \principal{} policy requires us to anticipate how these agents may respond to our policy choices, and entails two important challenges (Figure \ref{figure-intro-overview}).
First, the policies we choose induce a sub-game between the other agents, a sub-game which may have infinite equilibria.
The policy we choose should perform well regardless of which equilibria the agents respond with.
Second, \principal{} policies that perform well against rational \agentplayers{} may not generalize to boundedly rational \agentplayers{}, even if they are only infinitesimally irrational \cite{pita_robust_2010}.
Our policy should perform well even if agents act boundedly rational.

\begin{figure}[ht!]
 \includegraphics[width=\linewidth]{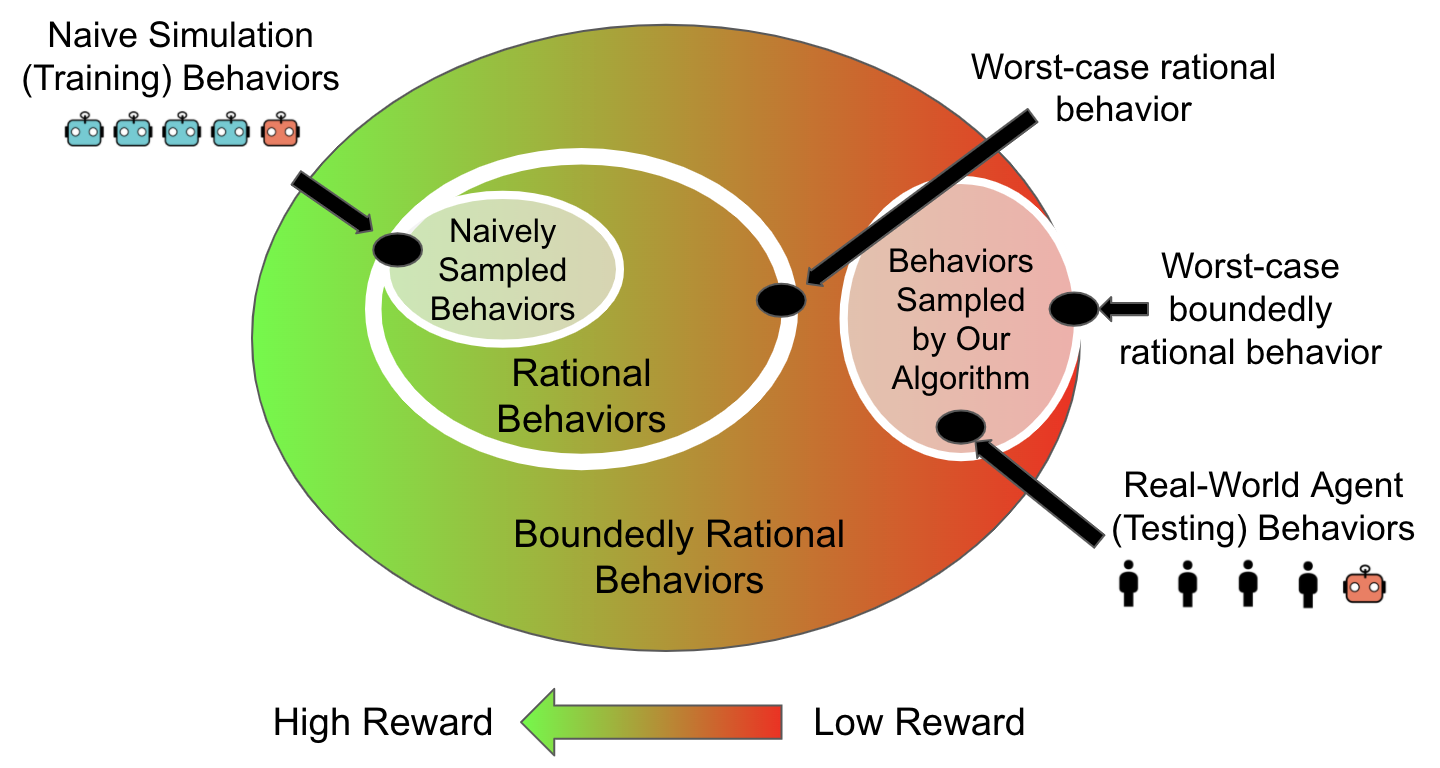} 
 \vspace{-15pt}
    \caption{
    The orange robot denotes our \principal{}, blue robots the \subagents{} we train against, and human icons the \subagents{} we encounter during test time.
    Evaluating a policy in a multi-agent game by naively sampling rational responses from other agents, e.g. via multi-agent RL, may lead to overly optimistic reward estimates.
    We introduce efficient algorithms for adversarially sampling rational responses in smooth games.
    These algorithms can be extended to sample worst-case boundedly rational responses (bottom-right).
    }
    \label{figure-intro-overview}
\end{figure}

We introduce a framework for the reinforcement learning of robust \principal{} policies that address these two challenges.
This framework evaluates a potential policy by identifying the worst-case coarse-correlated equilibrium (CCE) of the sub-game the policy induces.
Although identifying \emph{worst-case} CCE is generally computationally intractable~\cite{papadimitriou_computing_2008,barman_finding_2015}, we prove that worst-case CCE can efficiently learned in smooth games.
Our framework easily extends to identify worst-case approximate CCE.
This allows us to learn \principal{} policies that are robust to boundedly rational \subagents{}, such as agents whose incentives differ from the agents we train against.

Our motivating application is \emph{mechanism design} (``reverse game theory''), where a \principal{} implements the rewards and dynamics (the ``mechanism'') that other agents optimize for \cite{myerson_mechanism_2016}. 
Traditional mechanism design has been limited to problems with a convenient mathematical structure, e.g., simple auctions, where the equilibria behavior of agents can be solved in closed-form.
Recent research have pursued computational approaches to mechanism design that evaluate potential mechanisms using agent-based modeling \cite{holland_artificial_1991,bonabeau_agent-based_2002,dutting_optimal_2019} and multi-agent reinforcement learning (MARL) \cite{zheng_ai_2020}.
This application of multi-agent learning remains an exciting but understudied topic.
We will outline a modern perspective that formalizes automated mechanism design and its robustness concerns as an equilibrium selection problem.

\paragraph{Summary of results.} 
\begin{enumerate}
    \item \emph{We motivate a robust learning objective} for finding a principal policy that is robust to \subagents{} of differing incentives, bounded rationality, and reputation. This objective is a multi-follower extension of robust Stackelberg games.
    \item \emph{We show the existence of poly-time algorithms} for adversarially sampling the coarse-correlated-equilibria (\CCE{}) of smooth games, proving that multi-follower Stackelberg games can be tractable.
    This weakens prior findings that learning welfare-maximizing CCE is NP-hard \cite{papadimitriou_computing_2008}.
    \item \emph{We apply our proposed framework to automated mechanism design} problems where multi-agent RL is used to simulate the outcomes of mechanisms. In the spatiotemporal economic simulations used by the AI Economist \cite{zheng_ai_2020}, our framework learns robust tax policies that improve welfare by up to 15\% and are robust to previously unseen and boundedly rational agents.
\end{enumerate}

\section{Related Work}
\paragraph{Finding Equilibria.} 
A goal of multi-agent learning is finding equilibria (or more generally \emph{solution concepts}), i.e., sets of \agentplayer{} policies that are game-theoretically optimal (according to the definition of equilibrium). 
Prior work has used gradient-based methods, e.g., deep reinforcement learning, to find (approximate) equilibria with great success in multiplayer games such as Diplomacy \cite{gray_human-level_2021} and in training Generative Adversarial Networks \cite{schafer_implicit_2020}.
However, learning \emph{robust} \principals{} (or mechanisms) in multi-agent general-sum games is an open problem: it requires evaluating strategies against worst-case sub-game equilibria, which is computationally hard.

\paragraph{Stackelberg Games.}
Our \robustagent{} can be seen as a Stackelberg \emph{leader}; the other agents as Stackelberg \emph{followers}.
Stackelberg games have had real-world success in security games used in airports \citep{pita_using_2009} and anti-poaching defense \citep{nguyen_capture_2016}, for example.
However, Stackelberg analysis is typically limited to a single rational follower or assuming that the followers do not strategically interact, e.g., as in multi-follower security games \citep{korzhyk_security_2011}.
In contrast, we consider more general settings with multiple followers who may strategically interact with one another. 
In these settings, finding multi-follower ``best-responses'' is a computationally hard equilibria search and an open problem \cite{barman_finding_2015, basilico_bilevel_2020} for which our work provides a tractable perspective.

Our approach to modeling uncertainty about followers is inspired by \emph{Strict Uncertainty Stackelberg Games} that assume a worst-case choice of a follower's utility function \cite{pita_robust_2010,pita_robust_2012,kiekintveld_security_2013}.
Similarly, \emph{Bayesian Stackelberg Games} assume a prior over a space of possible follower utility functions \cite{conitzer_computing_2006}.
We extend this by considering a general setting with multiple, possibly interacting, followers.

\paragraph{Robustness.}
Our work is also related in spirit to prior work on robust reinforcement learning, which we discuss further in Appendix~\ref{app:relatedwork}.
We will refer to our notion of robustness as \emph{strategic robustness} to contrast it from non-game-theoretic notions of robustness to, for example, noisy observations \cite{morimoto_robust_2001}.
Strategic robustness also differs from the topic of ``robust game theory'', which studies the equilibria that arise when all players act robustly to some uncertainty about game parameters \cite{aghassi_robust_2006}.
Hereafter, we will simply refer to our notion of strategic robustness as ``robustness'' for brevity.

\section{Problem Formulation}
\paragraph{Notation.}
Bold variables are vectors of size $n$, with each component corresponding to an \agentplayer{} $i = 1, \dots, \numplayers$.
For example, $\bfaction \in \actionsetoneton$ denotes an action vector over all \agentplayers{} except our principal, \agentplayer{} $0$. $\bfaction_{-i}$ denotes the profile of actions chosen by all \agentplayers{} except \agentplayers{} $0$ and $i$.

\paragraph{Setup.}
We consider a general-sum game $\game$ with $\numplayers{} + 1$ agents. 
Our principal, or ``ego agent'', is index $i=0$ and the \otheragents{}{} are $i=1, \ldots, \numplayers$.
$\actionset_i$ denotes the set of $m_i$ actions available to \agentplayer{} $i$, and $m = \sum_{i=1}^\numplayers m_i$. 
$\pdfover{\actionset_i}$ is the set of probability distributions over action set $\actionset_i$.
The \emph{joint action set} is $\actionsetoneton := \prod_{i \in [1, \dots, \numplayers]} \actionset_i$.
$\pdfover{\actionsetoneton}$ denotes the set of joint distributions over strategy profiles $\actionsetoneton$.
$\prodpdfover{\actionsetoneton}$ denotes the set of product distributions over strategy profiles $\actionsetoneton$. 
Every \agentplayer{} $i = 0, \dots, \numplayers$ has a utility function $\utility_i: \actionset_0 \times \actionsetoneton \rightarrow \reals$ with bounded payoffs.
For example, $\utility_i(\action_0, \bfaction)$ denotes the utility of \agentplayer{} $i$ under action $\action_0 \in \actionset_0$ by our principal and actions $\bfaction \in \actionset$ by \agentplayers{} $1, \dots, \numplayers$.
When $\action_0$ is clear from context, we'll write $\utility_i(\bfaction)$, suppressing $\action_0$.
We denote expected utility as $\exputility_i(\policy_0, \vpolicy) := \EE_{\action_0 \sim \policy_0, \bfaction \sim \bfstrat} \utility_i(\action_0, \bfaction)$; again we write $\exputility_i(\bfstrat)$ when $\policy_0$ is clear from context.

\paragraph{Succinct Games.}
To derive our complexity results, we will use standard assumptions on our game $\game$ so that working with equilibria is not trivially hard \cite{papadimitriou_computing_2008}.
We assume $G := (I, T, U)$ is a succinct game, i.e., has a polynomial-size string representation. Here, $I$ are efficiently recognizable inputs, and $T$ and $U$ are polynomial algorithms. $T$ returns $n, m_0, \dots, m_\numplayers$ given inputs $z \in I$, while $U$ specifies the utility function of each agent.
We assume $\game$ is of \emph{polynomial type}, i.e., $n, m_0, \dots, m_\numplayers$ are polynomial bounded in $|I|$. 
We assume that $\game$ satisfies the \emph{polynomial expectation property}, i.e., utilities $\overline{u}(\policy_0, \vpolicy)$ can be computed in polynomial time for product distributions $\bfstrat$. 
The latter assumption is known to hold for virtually all succinct games of polynomial type \cite{papadimitriou_computing_2008}.
Without these assumptions, simply evaluating the payoff of a coarse-correlated equilibria can require exponential time.
All complexity results in our work, including Theorem \ref{thm:mainthm} and cited results from prior works use these assumptions.
Later, we will use additional ``smoothness'' conditions on $\game$ to overcome prior hardness results about succinct games.

\paragraph{Problem.}
Given a game $\game$, we aim to learn a principal policy $\policy_0 \in \pdfover{\actionset_i}$ that maximizes our principal's expected utility $\exputility_0(\policy_0, \bfstrat)$.
Here, $\bfstrat$ is the strategy that agents in a test environment respond to our policy $\policy_0$ with.
We will assume access to a training environment for $\game$.
In this training environment, we assume access to---potentially inaccurate---estimates of the reward functions of the agents; we will write these estimates as $\utility_1, \dots, \utility_\numplayers$.
For example, the \robustagent{} $i=0$ may be a policymaker setting a tax policy $\policy_0$ that maximizes social welfare $\utility_0$.
In response, the tax-payers play $\bfstrat$, choosing whether to work and report income.

\paragraph{Objective}
In order to formalize our robust learning objective, we must define what uncertainty sets $\behavior$ we want learning guarantees for.
A behavioral uncertainty set $\behavior(\policy_0) \subseteq \pdfover{\actionset}$ defines the strategies that the test environment agents may respond to a principal policy $\policy_0$ with.
For simple agent behaviors, one can use imitation learning or domain knowledge to construct these uncertainty sets.
In this work, we will study game-theoretic uncertainty sets, for example, where $\behavior(\policy_0)$ is the set of rational behaviors (see Section~\ref{sec:theory}).
Fixing a choice of $\uncertaintyset$, we can write our robust learning objective as:
\begin{align}
    \label{eq:nested_opt}
    \max_{\policy_0} \min_{\vpolicy \in \uncertaintyset(\policy_0)} \exputility_0(\policy_0, \bfstrat).
\end{align}
This is a challenging objective: it features nested optimization and requires searching over behavioral uncertainty sets, which is a non-trivial task in complicated games.

\section{Finding Uncertainty Sets for Boundedly Rational Agents}
\label{sec:theory}
A key challenge in our problem formulation is defining our behavioral uncertainty set $\uncertaintyset$.
In this section, we will first argue for a uncertainty set of \emph{coarse-correlated equilibria (\CCE{})}.
We will then prove that in \emph{smooth games} we can efficiently approximate worst-case \CCE{}.
We will finally propose a more relaxed uncertainty set that yields robustness to boundedly rational agents.

\paragraph{Initial Assumptions}
Our guiding principle for choosing $\uncertaintyset$ is that a robust principal can assume that agents are rational, but should still perform if agents act somewhat irrationally or have incentives that slightly differ from anticipated.
We will further assume the following and relax them later.
\begin{enumerate}[nosep]
    \item The incentives (reward functions) of the test agents exactly agree with our training environment's estimates. %
    \item All agents are rational expected-utility maximizers. No agent will settle if they want to unilaterally deviate.
    \item We, the principal, can commit to a strategy $\policy_0$ and will not react to other \agentplayer{}s.
\end{enumerate}

\paragraph{Dominant Strategies.} 
A natural choice of $\uncertaintyset$ is to extend Stackelberg equilibria to a multi-follower setting and define $\uncertaintyset(\policy_0)$ as the set of best responses to $\policy_0$,
\begin{align}
    \label{eq:stack_br}
    \uncertaintyset(\policy_0) = \{ \bfstrat \mid \forall i \in [1, \dots, n], \tilde{\bfstrat}_{-i} \in \pdfover{\actionset}_{-i}: \nonumber \\
    x_i \in \argmax_{x_i} \exputility_i(\policy_0, x_i, \tilde{\bfstrat}_{-i}) \}.
\end{align}
However, this set is only non-empty when all agents have dominant strategies: a strong assumption that rarely holds when followers interact with one another.

\paragraph{Stability-Based Equilibria.} 
We can also define an uncertainty set $\uncertaintyset(\policy_0)$ as the stable equilibria that agents may converge to under our policy $\policy_0$.
This coincides with Equation \ref{eq:stack_br} when it is non-empty.
Formally, let
\begin{align*}
    \uncertaintyset(\policy_0) = \{ \bfstrat \in \pdfover{\actionset} \mid (\policy_0, \bfstrat) \in \equilibrium \}, 
\end{align*}
where natural choices for $\equilibrium$ include mixed Nash equilibria (MNE) in which agents do not coordinate:
\begin{align*}
    \MNE = \{& \bfstrat \in \prodpdfover{\actionset} \mid  \forall i \in [1,\dots,\numplayers], \tilde{\action}_i \in \actionset_i: \nonumber \\ & \EE_{\bfaction \sim \bfstrat}[\utility_i(\vaction)] \geq \EE_{\bfaction_{-i} \sim \bfstrat_{-i}}[\utility_i(\tilde{\action}_i, \vaction_{-i})] \},
\end{align*}
or more general coarse-correlated equilibria (CCE):
\begin{align*}
    \CCE = \{ & \bfstrat \in \pdfover{\actionset} \mid   \forall i \in [1, \dots, \numplayers], \tilde{a}_i \in \actionset_i: \nonumber \\ & \EE_{\bfaction \sim \bfstrat}[\utility_i(\bfstrat)] \geq \EE_{\bfaction_{-i} \sim \bfstrat_{-i}}[\utility_i(\tilde{a}_i, \bfstrat_{-i})] \}.
\end{align*}
Here, coarse-correlated equilibria describe more general joint strategies, such as coordination based on shared information.

\paragraph{Computational Hardness.}
Unfortunately, optimizing the robustness objective in Equation \ref{eq:nested_opt} is neither tractable with \MNE{} nor \CCE{}. 
Finding the \MNE{}/\CCE{} that minimizes a utility function $\utility_0$ is equivalent to the NP-hard problem of finding a \MNE{}/\CCE{} that maximizes a linear social welfare objective $\specutility$ \cite{daskalakis_complexity_2009,papadimitriou_computing_2008}; we will set $\specutility = - \utility_0$ for convenience.
Beyond maximizing $\specutility$, simply finding a CCE that does not minimize $\specutility$ is NP-hard \cite{barman_finding_2015}.
Formally, consider the decision problem $\decproblem$ of determining whether a game $G$ (under our assumptions) admits a CCE $\bfstrat$ such that the expectation of $\specutility$, $\overline{\specutility}$, satisfies:
\begin{align*}
    \overline{\specutility}(\bfstrat) > \min_{\tilde{\bfstrat} \in \CCE} \overline{\specutility}(\tilde{\bfstrat}).
\end{align*}
This problem is NP-hard for some choices of $\specutility$, including the social welfare function \cite{barman_finding_2015}.
For our purposes, this implies that even sampling an approximately worst-case equilibria is intractable.
This means it could be impossible to efficiently evaluate our \robustagent{}'s \ourpolicy{} as it is intractable to guarantee sampling anything other than uninformative equilibria behavior.

\subsection{Smooth Games and Tractable Uncertainty Sets}
Smooth games offer a workaround to this hardness result.

\begin{definition}[Smooth Games]
\label{def:poa}
    A \emph{cost-minimization game} with cost functions $c_i$ and objective $C$ is $(\lambda, \mu)$-smooth if, for all strategies $\bfstrat, \bfstrat^* \in \pdfover{\actionset}$,
    \begin{align*}
        \sum_{i=1}^\numplayers \EE[c_i(x_i^*, \bfstrat_{-i})] \leq \lambda  \cdot \EE[C(\bfstrat^*)] + \mu \cdot  \EE[C(\bfstrat)].
    \end{align*}
    The ``robust price of anarchy'' (RPOA) is defined $\rho := \frac{\lambda}{1 - \mu}$.
\end{definition}

In fact, we can sample a $\CCE$ that approximately maximizes $\specutility= -\utility_0$ with run-time polynomial in the game size, smoothness $(\lambda_\game, \mu_\game)$ of the original game $\game$ and the smoothness $(\lambda_{\tilde{\game}}, \mu_{\tilde{\game}})$ of a modified game $\tilde{\game}$. Here, $\tilde{\game}$ is identical to $\game$ except each agent's utility is changed from $\utility_i$ to $\specutility$.

\begin{restatable}{theorem}{mainthm}
\label{thm:mainthm}
For succinct $n$-agent $m$-action games of polynomial type and expectation property, there exists a Poly($1/ \epsilon, n, m, \rho$) algorithm that will find an $\epsilon$-CCE $\bfstrat$ with
\begin{align*}
    \overline{\specutility}(\bfstrat) \geq \frac{y}{\rho}  - \epsilon.
\end{align*}
for any $y \leq \max_{\bfstrat^* \in \CCE} \overline{\specutility}(\bfstrat^*)$, where $\rho = \frac{\lambda_{\tilde{\game}}}{1 - \max\{ \mu_\game, \mu_{\tilde{\game}}\}}$.
\end{restatable}

\begin{proof}[Proof Sketch of Theorem \ref{thm:mainthm}]
First, we observe that the problem of finding a $\specutility$-maximizing $\CCE$ reduces to finding a halfspace oracle that optimizes some modified social welfare (Lemma~\ref{lemma:blackwell}).
Similar reductions have been described by \cite{jiang_general_2011} and \cite{barman_finding_2015}. 

\begin{restatable}{lemma}{blackwell}
\label{lemma:blackwell}
Fix a $0 < y \leq \max_{\bfstrat \in \CCE} \overline{\specutility}(\bfstrat)$.
Assume there is a Poly($1/\epsilon, n, m$)-time halfspace oracle that, given a vector $\beta \in \reals_+^{1 + m}$ with non-negative components, returns an $\bfstrat \in \pdfover{\actionset}$ such that $\beta v(\bfstrat) \leq 0$, where
    \begin{align}
    \label{eq:v}
        v(\bfstrat) &=\begin{bmatrix}
    \exputility_1(1, \bfstrat_{-1}) - \exputility_1(\bfstrat)  \\
    \vdots \\
    \exputility_1(m_1, \bfstrat_{-1}) - \exputility_1(\bfstrat)  \\
    \vdots \\
    \exputility_\numplayers(1, \bfstrat_{-\numplayers}) - \exputility_\numplayers(\bfstrat)  \\
    \vdots \\
     \exputility_\numplayers(m_\numplayers, \bfstrat_{-\numplayers}) - \exputility_\numplayers(\bfstrat)\\
    y - \overline{\specutility}(\bfstrat)
    \end{bmatrix}
    \end{align}
Then, there is a Poly($1 / \epsilon, n, m$)-time algorithm that returns an $\epsilon$-CCE $\bfstrat$ with $\overline{\specutility}(\bfstrat) \geq y - \epsilon$.
\end{restatable}

The halfspace oracle's optimization task can be reduced to optimizing social welfare in a game with a RPOA of $\rho$, which inherits its smoothness from games $\game$ and $\tilde{\game}$.
This upper bounds the ratio between the smallest objective value $\specutility$ of a CCE and the largest objective value $\specutility$ of any strategy. 
Thus, we can use no-regret dynamics \cite{cesa-bianchi_prediction_2006, FOSTER199740,hart} in this smooth game to approximate the half-space oracle.

\begin{restatable}{lemma}{halfspace}
\label{lemma:halfspace}
Let $\rho$ denote an upper bound on the price of anarchy of $\game, \tilde{\game}$.
There exists a Poly($1 / \epsilon, n, m, \rho$)-time halfspace oracle that, given a vector $\beta \in \reals_+^{1 + m}$ and $y \leq \rho \max_{\bfstrat \in \CCE}  \overline{\nu}(\bfstrat)$, returns an $\bfstrat \in \pdfover{\actionset}$ such that $\beta v(\bfstrat) \leq \epsilon$ where $v$ is defined as in Equation \ref{eq:v}.
\end{restatable}

Combining Lemmas \ref{lemma:blackwell} and \ref{lemma:halfspace} constructs our algorithm.
\end{proof}

Informally, this theorem states that it is tractable to find a CCE that maximizes $\specutility$ up to the price of anarchy.
While this relationship is immediate when $\specutility$ is the social welfare function \cite{roughgarden_intrinsic_2015}, our result shows that we can prove a similar relationship concerning the optimization of CCE against any linear function.
This positive result allows for translation between well-known price-of-anarchy bounds and bounds on the tractability of CCE optimization. %
\begin{restatable}{corollary}{strongp}
\label{corollary:strongp}
In linear congestion games, for any linear function $\specutility$, we can find, in Poly($1/\epsilon, n, m$) time, an $\epsilon$-CCE $\bfstrat$ such that
$
    \overline{\specutility}(\bfstrat) \geq 0.4 \cdot  \max_{\bfstrat \in \CCE} \overline{\specutility} (\bfstrat) - \epsilon
$.
\end{restatable}

While \citet{barman_finding_2015} showed the decision problem $\decproblem$ of finding a non-trivial CCE is NP-hard in general, our Theorem \ref{thm:mainthm} also shows it is tractable in smooth games. %
\begin{restatable}{corollary}{inp}
\label{corollary:inp}
The decision-problem $\decproblem$ is in P for games where $\frac{1}{\rho} \max_{\bfstrat \in \CCE} \overline{\specutility}(\bfstrat) > \min_{\bfstrat \in \CCE} \overline{\specutility}(\bfstrat)$. %
\end{restatable}

\paragraph{Remarks.}
These theoretical conclusions suggest that although using equilibria-based uncertainty sets may be intractable in some cases, there is a broad class of common problems where CCE uncertainty sets are reasonable and allow for efficient adversarial sampling.
The algorithm we construct in Theorem~\ref{thm:mainthm} also enjoys two nice properties.
First, it only requires oracle access to utility functions (efficient under polynomial expectation property).
Second, in the algorithm's self-play subprocedures, each agent can be trained using only their, and their principal's, utility information. %

\subsection{Weakening Assumptions on Agents and Principal Robustness Algorithm}
We now further refine our choice of uncertainty set to ensure generalization to agents that violate our behavioral assumptions.
We now switch to weaker assumptions:
\begin{enumerate}
    \item \textbf{Subjective rationality:} At test time, an agent's utility $\tilde{\utility}_i$ may differ from the anticipated utility $\utility_i$ \cite{simon_substantive_1976}.
    Many models of subjective rationality, such as Subjective Utility Quantal Response \cite{nguyen_analyzing_2013}, bound the gap between $\utility_i$ and $\tilde{\utility}_i$ as $\norm{\utility_i - \tilde{\utility_i}}_\infty \leq \gamma_s$ with $\gamma_s > 0$.

    \item \textbf{Procedural rationality:} An \agentplayer{} may not fully succeed in maximizing their utility \cite{simon_substantive_1976}, e.g.,
    they could gain up to $\gamma_p > 0$ utility if they unilaterally deviate.
    
    \item \textbf{Myopia:} An \agentplayer{} may possess commitment power or otherwise be non-myopic, factoring in long-term incentives with a discount factor $\gamma_m \in (0,1)$. %
    This relates to notions of exogenous commitment power, e.g., partial reputation, in Stackelberg games \cite{kreps_reputation_1982, fudenberg_reputation_1989}.
\end{enumerate}

These variations represent common forms of bounded rationality.
We now show that the sampling scheme we devise for Theorem 1 can be extended to maintain robustness despite these weaker assumptions.
We aim to learn strategies $\policy_0$ that perform well even when presented with agents possessing these variations. 
Hence, we aim to use uncertainty sets $\uncertaintyset$ that encode such behaviors.
The next proposition suggests it suffices to simply relax our uncertainty set $\uncertaintyset$ to include more approximate equilibria. %
\begin{restatable}{proposition}{generalization}
\label{prop:generalization}
The uncertainty set $\uncertaintyset'$ of (any combination of) agents violating assumptions 1-3 with parameters $\gamma_m, \gamma_s, \gamma_p$ is contained in the set of $\varepsilon$-\CCE:
\begin{align*}
    \CCE_\varepsilon = \{ & \bfstrat \in \pdfover{\actionset} \mid   \forall i \in [1, \dots, \numplayers], \tilde{a}_i \in \actionset_i: \\ & \EE_{\bfaction \sim \bfstrat}[\utility_i(\vaction)] + \varepsilon_i \geq \EE_{\bfaction_{-i} \sim \bfstrat_{-i}}[\utility_i(\tilde{a}_i, \vaction_{-i})] \}, \nonumber
\end{align*}
where $\varepsilon_i = \max\{\frac{\norm{\utility_i}_\infty}{1-\gamma_m}, \gamma_s, 2 \gamma_p\}$.
Hence, we can train policies robust to such agents by using the following in the robustness objective of Equation \ref{eq:nested_opt}:
\begin{align*}
    \uncertaintyset_\varepsilon(\policy_0) = \{ \bfstrat \in \pdfover{\actionset} \mid \exists \policy_0 \in \actionset_0: (\policy_0, \bfstrat) \in \CCE_\varepsilon\}.
\end{align*}
\end{restatable}

This proposition motivates us to use approximate CCE as an uncertainty set rather than exact CCEs.
Conveniently, the algorithm we construct in the proof of Theorem~\ref{thm:mainthm} can be modified to adversarially sample from $\epsilon$-CCE instead of exact CCE.
By relaxation of Lemma \ref{lemma:blackwell}, it yields the same optimality and runtime guarantees as Theorem \ref{thm:mainthm}, but over the set of approximate $\epsilon$-CCE.
We will refer to this modification of the Theorem~\ref{thm:mainthm} algorithm as Algorithm \ref{alg:theoryinner}, which we repeat in full in the Appendix.

\section{Finding Approximate Uncertainty Sets with Blackbox Optimizers}
\label{sec:algorithm}
One challenge with using Algorithm \ref{alg:theoryinner} in practice is that it relies on a no-regret learning subprocedure that does not scale well.
This is a common bottleneck in multi-agent learning when deriving practical algorithms from algorithms with strong theoretical guarantees.
A common remedy is to replace the no-regret learning procedure with a standard learning algorithm (e.g., SGD), usually with no negative impact on empirical behavior or performance.
We will do exactly this to derive a practical variant of Algorithm \ref{alg:theoryinner} that still inherits its theoretical intuitions.
This involves two main steps.
\paragraph{Removing binary search.}
Algorithm \ref{alg:theoryinner}'s binary search over $y$ is a theoretically efficient search over possible optimal or worst-case values of $\epsilon$-CCE.
However, it is inefficient in a nested optimization like Equation \ref{eq:nested_opt}. 
Observing that $y$'s value affects only the parameterization of the importance weight $\beta_{m+1}$, we can fix $v_{m+1}$ to a sufficiently small value such that $\beta_{m+1} = 1[v_1, \dots, v_m \leq 0]$.

\paragraph{Replacing Blackwell's algorithm.}
During Algorithm \ref{alg:theoryinner}'s reduction to halfspace oracles, we can merge components of $v$ corresponding to the same agent.
This yields a functional equivalent of Eq \ref{eq:v},
\begin{align}
\label{eq:newv}
    v(x) = \begin{bmatrix}
    \max_{a_1 \in \actionset_1} \overline{\utility}_1(a_1, \bfstrat_{-1}) - \overline{\utility}_1(\bfstrat) \\
    \vdots \\
    \max_{a_n \in \actionset_n} \overline{\utility}_n(a_n, \bfstrat_{-n}) - \overline{\utility}_n(\bfstrat) \\
    \sigma \rightarrow 0
    \end{bmatrix}.
\end{align}
In practice, this choice is more tractable than Eq \ref{eq:v}.
Eq \ref{eq:newv} lends itself to many efficient approximations.
For instance, when $\actionset$ is combinatorially large, we can approximate the regret estimates with local methods rather than explicitly enumerating all possible action deviations.

\paragraph{A practical algorithm.}
Combined, these two modifications to Algorithm \ref{alg:theoryinner} render it equivalent to computing an upper-bound for the Lagrangian dual problem, $L(\epsilon)$, of adversarial sampling (Equation \ref{eq:adv_sample_lp}),
\begin{align}
\label{eq:lagdual}
   L(\epsilon) = & \min_{\bfstrat \in \pdfover{ \actionset}} \max_\lambda
        \overline{\utility}_0(\bfstrat) - \sum_{i=1}^\numplayers \lambda_i \brcksq{ \text{Reg}_i\brck{\bfstrat} - \epsilon},
   \\
   & \text{Reg}_i\brck{\bfstrat} := \max_{a_i \in \actionset_i} \overline{\utility}_i(a_i, \bfstrat_{-i}) - \overline{\utility}_i(\bfstrat). \nonumber
\end{align}
Here, decoupled no-regret dynamics efficiently approximate the outer optimization $\min_{\bfstrat \in \pdfover{\actionset}}$.
In this sense, the theoretical results of Section~\ref{sec:theory} can be interpreted as a formal bound on how much decoupled approximations of $\min_{\bfstrat \in \pdfover{\actionset}}$ affect the lower-bound of this dual problem.
Theorem \ref{thm:mainthm} can thus be interpreted as the implication that, when playing a sufficiently smooth game, it is reasonable to use decoupled algorithms to approximate the dual problem.
This motivates our final modification to Algorithm~\ref{alg:theoryinner}: replacing the inner no-regret learning loop with a blackbox self-play algorithm.
The final algorithm is described in Algorithm~\ref{alg:blackboxinner}.

\subsection{Adversarial Sampling Experiments}
Before applying Algorithm~\ref{alg:blackboxinner} to more ambitious mechanism design tasks, we first benchmark the quality of its adversarial sampling.
As our eventual mechanism design applications are spatiotemporal games requiring multi-agent reinforcement learning (MARL), for this experiment, we will also use MARL as the blackbox self-play procedure of Algorithm~\ref{alg:blackboxinner}. 
In particular, we will use a common multi-agent implementation of the PPO algorithm \cite{schulman_proximal_2017} and a Monte-Carlo sampling scheme as our regret estimator.

\begin{algorithm}[t]
\caption{Decoupled sampling of pessimistic equilibria.}
\label{alg:blackboxinner}
\begin{algorithmic}
\STATE \textbf{Output:} Approximate lower-bound on $L(\epsilon)$ (Eq \ref{eq:lagdual}).
\STATE \textbf{Input:} Number of training steps $M_\text{tr}$ and self-play steps $M_s$, reward slack $\epsilon$, multiplier learning rate $\alpha_{\lambda}$, uncoupled self-play algorithm $B$, regret estimators $R_i: \pdfover{\actionset} \rightarrow \reals$ for each agent $i$.
\STATE Initialize mixed strategy $\bfstrat_1$.
\FOR{$j = 1, \dots, M_{\text{tr}}$}
\FOR{$i = 1, \dots, \numplayers$}
    \STATE Estimate regret $r_i$ as $\hat{r}_i \leftarrow R_i(\bfstrat_{j})$, where $r_i := \max_{\tilde{x}_i \in \pdfover{\actionset_i}}  \overline{\utility}_i(\tilde{x}_i, \bfstrat_{-i}) - \overline{\utility}_i(\bfstrat). \nonumber$
    \STATE Compute multiplier $\lambda_i \leftarrow \lambda_i - \alpha_{\lambda} \brck{\hat{r}_i - \epsilon}$.
\ENDFOR
\STATE Using $B$, run $M_s$ rounds of self-play with utilities
$\hat{\utility}_i(\bfaction) := (\lambda_i \utility_i(\bfaction) - \utility_0(\bfaction) ) / (1 + \lambda_i)$.
\STATE Set $\bfstrat_{j+1}$ as the resulting empirical play distribution.
\ENDFOR
\STATE Return $\frac{1}{M_{\text{tr}}} \sum_{t=1}^{M_{\text{tr}}} \overline{u}_0(\bfstrat_t)$. 
\end{algorithmic}
\end{algorithm}
\vspace{-0.05in}

\begin{figure*}[t!]
    \centering
    \includegraphics[width=0.4\linewidth]{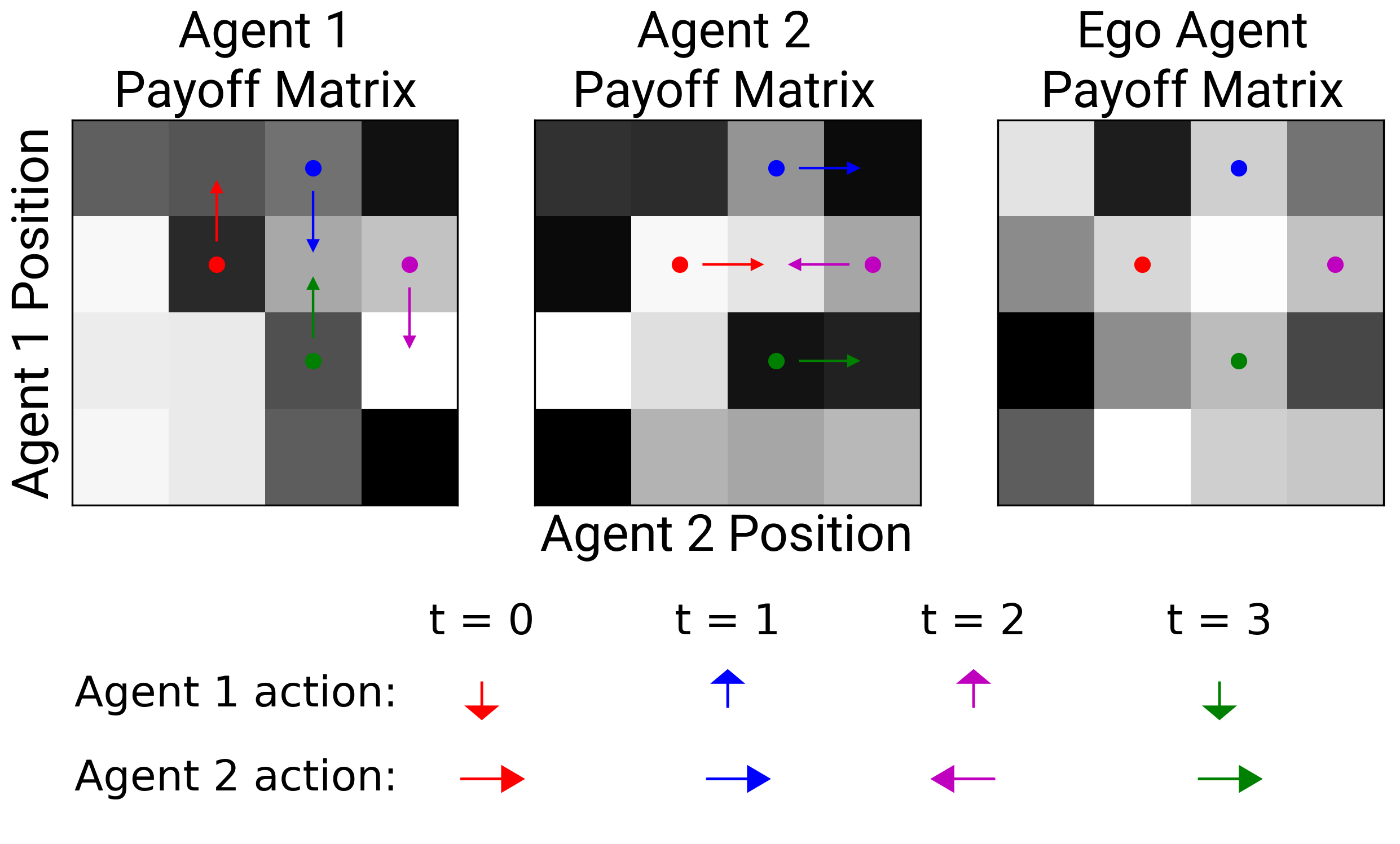}
    \includegraphics[width=0.59\linewidth]{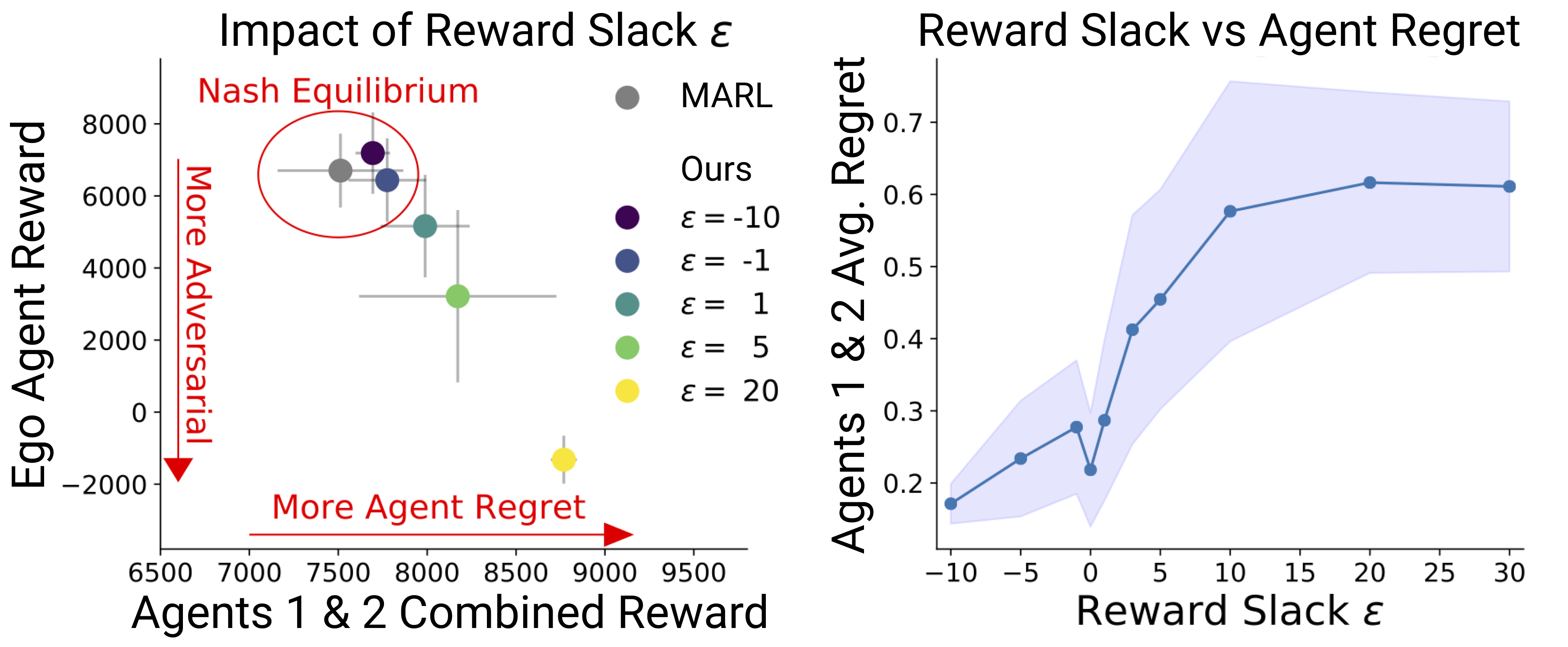}
    \caption{
        Validating \ouralgolong{} (Algorithm \ref{alg:dynamicmechanism}) in constrained repeated bimatrix games.
        \textbf{Left:} In the repeated bimatrix game, two agents navigate a 2D landscape. 
        Both agents and the principal receive rewards based on visited coordinates. Brighter squares indicate higher payoff.
        The bimatrix reward structure encodes a social dilemma featuring a Nash equilibrium with low reward for the two agents and high reward for the \bimatrixactor.
        Vanilla MARL converges to this equilibrium.
        \textbf{Right:} Agents trained with \ouralgo{} deviate from this equilibrium in order to reduce the reward of the principal.
        $\epsilon$ governs the extent of the allowable deviation.
        As $\epsilon$ increases, the average per-timestep regret experienced by the agents also increases. Each average is taken over the final 12 episodes after rewards have converged.
        Each point in the above scatter plots describes the average outcome at the end of training for the agents ($x$-coordinate) and the principal ($y$-coordinate). Error bars indicate standard deviation.
    }
    \label{fig:envs:bimatrix}
\end{figure*}

\paragraph{Game Environment.}
The game environment for this experiment is a Sequential Bimatrix Game.
This is an extension of the classic \textit{repeated bimatrix game} (Figure \ref{fig:envs:bimatrix}), whose Nash equilibria can be solved efficiently and is well-studied in game theory.
At each timestep $t$, a row (agent 1) and column player (agent 2) choose how to move around a $4\times 4$ grid, while receiving rewards $r_1(\state_i, \state_j), r_2(\state_i, \state_j)$.
The current location is at row $\state_i$ and column $\state_j$. 
The row (column) player chooses whether to move up (left) or down (right).
Each episode is $500$ timesteps.

We configure the payoff matrices $r_1$ and $r_2$, illustrated in Figure \ref{fig:envs:bimatrix}, so that only one Nash equilibrium exists and that the equilibrium constitutes a ``tragedy-of-the-commons,'' where agents selfishly optimizing their own reward leads to less reward overall.
The principal is a passive observer that observes the game and receives a payoff $\robustagentreward(\state_i, \state_j)$.
The principal does not take any actions and its payoff is constructed such that its reward is high when the agents are at the Nash equilibrium.
If Algorithm~\ref{alg:blackboxinner} successfully samples realistic worst-case behaviors, we expect to see agents 1 and 2 learning to deviate from their tragedy-of-the-commons equilibrium in order to (1) reduce the principal's reward but also (2) without significantly increasing their own regret.

\paragraph{Algorithm~\ref{alg:blackboxinner} efficiently interpolates between adversarial and low-regret equilibria.}
In Figure \ref{fig:envs:bimatrix} (middle), we see the equilibria reached by agents balance the reward of the principal (y-axis) and themselves (x-axis).
In particular, we see that conventional multi-agent RL discovers the attractive Nash equilibrium, which is in the top left.
At this equilibrium, the agents do not cooperate and the principal receives high reward.
Similarly, for small values of $\epsilon$, our algorithm discovers the Nash equilibrium.
Because $\epsilon$ acts as a constraint on agent regret, with larger values of $\epsilon$, \ouralgo{} deviates farther from the Nash equilibrium, discovering $\epsilon$-equilibria to the bottom-right that result in lower principal rewards.

\paragraph{Algorithm~\ref{alg:blackboxinner} has tight control over how much agents sacrifice to hurt the principal.}
We see in Figure \ref{fig:envs:bimatrix} (right) that deviations from the Nash equilibrium yield higher regret for the agents, i.e., regret increases with $\epsilon$.
This figure also confirms that increasing our algorithm's slack parameter correctly increases the incentive of the agents to incur regret in order to harm the principal.

\section{Optimizing Strategic Robustness}
\label{sec:nested}
\begin{figure*}[ht]
    \centering
    \small
    \begin{tabular}{lrrrrrrr}
\textbf{Training $\downarrow$ Testing $\rightarrow$} & \textbf{Original} & \textbf{\AdversarialAgent{} ($Q=0.25$)} & \textbf{\AdversarialAgent{} ($Q=1$)} & \textbf{RiskAv ($\eta=0.05$)} & \textbf{RiskAv ($\eta = 0.2$)} 
\\ \hline
MARL & 104\tiny{$\pm$50.5} & -5.8\tiny{$\pm$34.1} & -232\tiny{$\pm$29.3} & 383\tiny{$\pm$14.4} & 352\tiny{$\pm$42.6} \\
\AdversarialAgent{} ($Q = 0.25$) & \mmbf{143}\tiny{$\pm$35.0} & \mbf{64.7}\tiny{$\pm$35.1} & -191\tiny{$\pm$36.0} & 236\tiny{$\pm$23.8} & 292\tiny{$\pm$28.4} \\
\AdversarialAgent{} ($Q = 1$) & 131\tiny{$\pm$63.1} & -23\tiny{$\pm$10.1} & \mbf{-47}\tiny{$\pm$8.20} & 286\tiny{$\pm$35.1} & 290\tiny{$\pm$36.0} \\
RiskAv ($\eta = 0.05$) & -20\tiny{$\pm$44.2} & -112\tiny{$\pm$15.7} & -222\tiny{$\pm$29.3} & \mbf{404}\tiny{$\pm$47.2} & \mbf{464}\tiny{$\pm$0.95} \\
RiskAv ($\eta = 0.2$) & -53\tiny{$\pm$32.4} & -150\tiny{$\pm$20.0} & -283\tiny{$\pm$29.8} & \mbf{465}\tiny{$\pm$0.61} & \mmbf{358}\tiny{$\pm$70.9} \\
\ouralgoshort{} $\epsilon=-10$ & \mbf{227}\tiny{$\pm$50.8} & \mbf{48.9}\tiny{$\pm$12.4} & \mmbf{-137}\tiny{$\pm$43.1} & 265\tiny{$\pm$41.8} & 292\tiny{$\pm$38.3} \\
\ouralgoshort{} $\epsilon=50$ & \mbf{221}\tiny{$\pm$80.8} & \mbf{48.3}\tiny{$\pm$29.7} & \mbf{-53}\tiny{$\pm$14.0} & \mbf{460}\tiny{$\pm$33.9} & \mbf{481}\tiny{$\pm$38.6}
    \end{tabular}
    \caption{
    \textbf{Robust performance in $\numagents$-agent matrix games.} 
    We train an \robustagent{} in a $7\times7\times7\times7$ matrix game with $n=4$ agents (including the \robustagent{}) until convergence.
    For each method, we train 20 seeds and select the top 10 in a validation environment.
    Each row corresponds to a specific agent type that the \robustagent{} is trained on.
    'MARL' refers to agents trained using their `Original` reward definition;
    `\AdversarialAgent{}' refers to adversarial agents;
    `\RiskAverseAgent{}' refers to risk-averse agents.
    The \robustagent{}s trained on these types of agents tend to perform best when evaluated on the same type seen during training.
    In contrast, \robustagent{}s trained against agent behaviors sampled using \ouralgolong{} ($\epsilon = 50$) perform within standard error of top-1 on all agent types.
    We use the `Original` reward definition when training with \ouralgolong{}.
    }
    \label{fig:nmatrixperf}
\end{figure*}

\begin{figure*}[ht]
    \centering
    \small
    \begin{tabular}{lrrrrrrr}
\textbf{Training $\downarrow$ Testing $\rightarrow$} & \textbf{Original} & $\eta = 0.11$ & $\eta = 0.19$ & $\eta = 0.27$ & $\alpha = 0.25$ & $\alpha = 2.5$ \\ \hline
Free Market & 326\tiny{$\pm$1} & 527\tiny{$\pm$2} & 427\tiny{$\pm$1} & \mmbf{162}\tiny{$\pm$1} & 248\tiny{$\pm$2} & 112\tiny{$\pm$0} \\
Federal & 335\tiny{$\pm$8} & 637\tiny{$\pm$5} & 497\tiny{$\pm$2} & 150\tiny{$\pm$2} & \mbf{270}\tiny{$\pm$1} & 121\tiny{$\pm$0} \\
Saez & \mbf{381}\tiny{$\pm$1} & 597\tiny{$\pm$3} & 487\tiny{$\pm$4} & \mbf{189}\tiny{$\pm$1} & 265\tiny{$\pm$0} & 127\tiny{$\pm$0} \\
\hline
Ours ($\epsilon=-30$) & \mbf{375}\tiny{$\pm$9} & \mmbf{646}\tiny{$\pm$6} & \mmbf{514}\tiny{$\pm$12} & \mmbf{164}\tiny{$\pm$9} & \mmbf{266}\tiny{$\pm$2} & \mmbf{131}\tiny{$\pm$2} \\
AI Economist (Original) & \mbf{386}\tiny{$\pm$2} & 628\tiny{$\pm$5} & \mmbf{515}\tiny{$\pm$1} & 123\tiny{$\pm$13} & \mmbf{267}\tiny{$\pm$0} & 129\tiny{$\pm$1} \\
AI Economist ($\eta=0.11$) & 253\tiny{$\pm$5} & \mbf{683}\tiny{$\pm$7} & 506\tiny{$\pm$1} & 140\tiny{$\pm$1} & 255\tiny{$\pm$2} & 129\tiny{$\pm$0} \\
AI Economist ($\eta=0.19$) & 308\tiny{$\pm$17} & \mmbf{665}\tiny{$\pm$9} & \mbf{543}\tiny{$\pm$6} & 82\tiny{$\pm$29} & 256\tiny{$\pm$3} & \mmbf{131}\tiny{$\pm$1} \\
AI Economist ($\eta=0.27$) & 339\tiny{$\pm$11} & 603\tiny{$\pm$3} & 477\tiny{$\pm$1} & 137\tiny{$\pm$10} & \mmbf{266}\tiny{$\pm$0} & 129\tiny{$\pm$0} \\
AI Economist ($\alpha=0.25$) & 324\tiny{$\pm$2} & 625\tiny{$\pm$10} & 501\tiny{$\pm$7} & 121\tiny{$\pm$25} & 263\tiny{$\pm$0} & 128\tiny{$\pm$0} \\
AI Economist ($\alpha=2.5$) & 104\tiny{$\pm$27} & 636\tiny{$\pm$3} & 246\tiny{$\pm$33} & 49\tiny{$\pm$10} & 251\tiny{$\pm$4} & \mbf{135}\tiny{$\pm$1} \\
    \end{tabular}
    \caption{
    \textbf{Robust dynamic tax policies in a spatiotemporal economy}.
    First 3 rows are classic tax baselines: ``Free Market'' has no taxes;  ``US Federal'' uses the 2018 US Federal progressive income tax rates;
    ``Saez'' uses an adaptive, theoretical formula to estimate optimal tax rates.
    Bottom rows correspond to learned policies trained to optimize, and evaluated on, the social welfare metric \SWF{} of equality and productivity.
    `Ours' and `AI Economist (Original)' are trained on the `Original' settings (risk aversion $\eta = 0.23$; entropy bonus $\alpha = 0.025$).
    Naive multi-agent reinforcement learning tax policies, including \cite{zheng_ai_2020}'s original AI Economist, fail to generalize to previously unseen agent types.
    In contrast, our algorithm performs within standard error of top-1 on all agent types.
    }
    \label{fig:aieperf}
\end{figure*}

We now apply our adversarial sampling scheme, Algorithm~\ref{alg:blackboxinner}, to automated mechanism design problems.
In particular, we will use Algorithm~\ref{alg:blackboxinner} to provide feedback to a reinforcement learning (RL) procedure that selects mechanisms.
This RL procedure, Algorithm \ref{alg:dynamicmechanism}, is described in the Appendix for completeness.
First, we induce a mechanism design problem on a repeated $n$-matrix game.
Then, we'll seek to learn an optimal tax policy in the AI Economist, a large-scale spatiotemporal simulation of an economy \cite{zheng_ai_2020}.
Each experiment setting features 4 to 5 agents involved in complex multi-timestep interactions.
They are thus significantly more complex and costly to train in than traditional multi-agent RL environments.

\subsection{Repeated Matrix Games.}
We first extend our repeated bimatrix game to include additional players and a principal.

\paragraph{Setup.}
The setting is now a $4$-player, general-sum, normal-form game on a randomly generated $7 \times 7 \times 7 \times 7$ payoff matrix, with the same action and payoff rules as shown in Figure \ref{fig:envs:bimatrix}.
Recall that these players will engage one another for $500$ timesteps in an episode.
Each player $i$ is associated with an ``original'' reward function $\rew_i$; this is the reward function we will have access to during training.
During test time, we may also encounter other categories of agents that have different reward functions but who are also themselves learned with RL.
\begin{enumerate}
    \item \emph{Vanilla}: $\rew_i$.
    \item \emph{Adversarial} (Adv):
   $\rew'_i = \rew_i - Q \robustagentreward$; larger $Q$ is more adversarial.
    \item \emph{Risk-averse} (RiskAv): $
    \rew_i' = (\rew_i^{1 - \isoeta} - 1)/\brck{1 - \isoeta}$;
    higher $\isoeta$ is more risk averse.
\end{enumerate}

\paragraph{Results.}
Figure \ref{fig:nmatrixperf} shows the average rewards of principals trained (rows) and evaluated (columns) on each type of agents.
We observe three key trends.
First, principals trained on one type of agent generally perform better when evaluated on the same type of agent (diagonal entries).
Second, principals trained with our adversarial sampling scheme perform better across the board.
Third, the robustness gains of our adversarial sampling scheme are stronger when $\epsilon$ is large.
This is expected as $\epsilon$ parameterizes the adversarial strength of our sampling scheme.
For small $\epsilon$, adversarial sampling reduces to random sampling as the CCE constraint is so tight it permits no adversarial deviations.
We also saw that, even though all methods were run with 20 seeds and filtered down to 10 seeds on a validation set, \ouralgolong{}'s results remain somewhat noisy, as it may not converge when badly initialized.

\subsection{Taxing a Simulated Economy.}
We now apply \ouralgo{} to designing dynamic (multi-timestep) tax policies for a simulated trade-and-barter economy with strategic taxpayers that interact with one another  \cite{zheng_ai_2020}.
See Appendix \ref{sec:spatiotemporalgame} for a screenshot.

\paragraph{Setup.}
In this simulated economy, the \robustagent{} sets a tax policy and the agents play a partially observable game, given the tax policy.
Each episode is 1,000 timesteps of economic activity.
Taxpayers earn income $\income_{i,t}$ from labor $\labor_{i,t}$ and pay taxes $\tax(\income_{i,t})$. 
They optimize their expected isoelastic utility:
\begin{align}
& \posttax{\income}_{i,t} = \income_{i,t} - T(\income_{i,t}), \quad  \rew_{i,t}(\posttax{\endow}_{i,t}, \labor_{i,t}) = \frac{\posttax{\endow}_{i,t}^{1 - \isoeta} - 1}{1 - \isoeta} - \labor_{i,t},
\end{align}    
where $\posttax{\endow}_{i,t}$ is the post-tax endowment of agent $i$, and $\isoeta>0$ sets the degree of risk aversion (higher $\isoeta$ means higher risk aversion) \citep{arrow_theory_1971}.
Players expend labor and earn income by participating in a rich simulated grid-world with resources and markets.
The \robustagent{} optimizes for social welfare $\SWF{} = \brck{1 - \frac{\numagents}{\numagents - 1} \gini(\income)} \cdot \brck{\sum_{i=1}^N \income_i}$,
the product of equality \citep{gini_variabilita_1912} and productivity.
The taxpayers are also themselves learned with multi-agent RL, using a PPO algorithm entropy hyperparameter $\alpha$ \cite{schulman_proximal_2017}.

\paragraph{Results.}
Figure \ref{fig:aieperf} shows the social welfare achieved by \ouralgo{}, naive dynamic RL policies \cite{zheng_ai_2020}, and static baseline tax policies (Saez, US Federal).
Naive RL policies achieve good test performance when evaluated on the same agents seen in training, but perform poorly with agents with different $\eta$ and noise level. %
They are often outperformed by the baseline taxes, which perform surprisingly well under strong risk aversion ($\eta =0.27$) and noisy agents (entropy bonus $\alpha = 0.25, 2.5$).
We see that the static baseline taxes may be more robust than dynamic ones, even in complex environments.
However, \ouralgo{} closes this robustness gap, consistently outperforming or tying both AI Economists and baseline taxes.

\section{Future Work}
\label{sec:futurework}
Efficient sampling of worst-case equilibria is a key challenge for robust decision-making, and by extension, automated mechanism design.
As we've explored uncertainty sets based on game-theoretic concepts, future work may build uncertainty sets that use domain knowledge or historical data and that may yield robustness to other types of domain shifts, e.g., in game dynamics.

\newpage
\bibliography{main}

\newpage
\ %
\newpage
\appendix
\newpage
\section{Appendix: Proofs}

\subsection{Proof of Lemma \ref{lemma:blackwell}}
First, we observe that finding a $\specutility$-maximizing CCE is equivalent to solving the linear program below.
Writing the joint distribution $\bfstrat$ as a $\prod_{i=1}^\numplayers m_i$-size vector of joint densities, a $\specutility$-maximizing \CCE{} is given by:
\begin{align}
\label{eq:adv_sample_lp}
    \max_{\bfstrat \in \pdfover{\actionset}} \overline{\specutility}(\bfstrat) \text{  s.t.  } & \forall i \in [1,\dots, \numplayers], \forall \action_i \in \actionset_i: \\
    & \sum_{\bfaction' \in \actionset}  \left[ \utility_i(\action_i, \bfaction'_{-i}) - \utility_i(\bfaction') \right] \bfstrat(\bfaction') \leq 0.\nonumber
\end{align}
Although the number of constraints in this LP is polynomial, the number of variables is non-trivially exponential.
In addition, the polynomial expectation property only guarantees polynomial time expected utility computations for product distributions.

\blackwell*

\begin{proof}
    We will prove our lemma constructively.
    This proof closely mirrors \cite{cesa-bianchi_prediction_2006}'s proof of Blackwell approachability.
    
    First, we establish approachability conditions.
    Note that $\bfstrat$ is an $\epsilon$-CCE satisfying $\overline{\specutility}(\bfstrat) \leq \epsilon + y$ if $\bfstrat$ satisfies $v(\bfstrat) \leq \epsilon$ component-wise; the first $m$ components give that $\bfstrat$ is an $\epsilon$-CCE and the last component gives that $\overline{\specutility}(\bfstrat) \leq \epsilon + y$.
    Thus if $y \leq \max_{\bfstrat \in \CCE} \overline{\specutility}(\bfstrat)$, then we know there exists a $\bfstrat \in \pdfover{\actionset}$ such that $v(\bfstrat) \leq 0$.
    In other words, the closed convex set $v(\bfstrat) \leq 0$ is approachable.
    
    For the remainder of this proof, we assume $y \leq \max_{\bfstrat \in \CCE} \overline{\specutility}(\bfstrat)$ and prove that we will return a $\epsilon$-CCE that satisfies $\overline{\specutility}(\bfstrat) \leq \epsilon + y$.
    If $y$ does not satisfy this condition, then we will not return a CCE  $\bfstrat$ that satisfies $\overline{\specutility}(\bfstrat) \leq y$ because no such $\bfstrat$ would exist.
    Let $\bfstrat^* := \argmax_{\bfstrat \in \CCE} \overline{\specutility}(\bfstrat)$.
    
    Fix a choice of $\epsilon$.
    Now, consider the following iterative algorithm where we index timesteps by $t= 1,2,\dots$ and produce a sequence of mixed strategies beginning with some arbitrary choice of $\bfstrat^{(0)} \in \pdfover{\actionset}$.
    We now consider every timestep $t$ where the average value of $v(\bfstrat)$,
    \begin{align*}
        \overline{v}_t :=  \frac{1}{t-1} \sum_{i=1}^{t-1} v(\bfstrat^{(i)}),
    \end{align*}
    does not satisfy $\overline{v}_t \leq \epsilon$ component-wise.
    If instead $\overline{v}_t \leq \epsilon$ holds component-wise, we can return $\overline{\bfstrat}$, a uniform distribution over $\bfstrat^{(1)}, \dots, \bfstrat^{(t-1)}$, and be done.
    
    Otherwise, project $\overline{v}_t$ onto the negative orthant, $S := [-\infty, 0]^{1+m}$, to obtain a vector $\beta \in \reals^{m+1}$,
    \begin{align*}
        a_t &= \argmin_{a \in S} \norm{\overline{v}_t - a} \\
        \beta_t &= \frac{\overline{v}_t - a_t}{\norm{\overline{v}_t - a_t}}
        .
    \end{align*}
    Since $S$ is closed and convex, $\beta$ exists and is unique.
    Note that $\beta$ must be non-negative in all components by definition of $S$.
    We then consult our halfspace oracle to return our next iterate $\bfstrat^{(t)} \in \pdfover{ \actionset}$ where $\beta_t v(\bfstrat^{(t)}) \leq 0$.
    A solution must exist as, for any $t$, $\bfstrat^*$ satisfies $\beta_t v(\bfstrat^{*}) \leq 0$.
    
    We now prove this algorithm returns the desired output in Poly($\frac{1}{\epsilon}, n, m$) and hence Poly($\frac{1}{\epsilon}, |I|$) time.
    
    Since $\overline{v}_{t+1} = \frac{t-1}{t} \overline{v}_{t} + \frac{1}{t} v(\bfstrat^{(t)})$, we may write,
    \begin{align*}
        d(\overline{v}_{t+1}, S)^2 & = \norm{\overline{v}_{t+1}-a_{t+1}}^2 \\
         \leq & \norm{\overline{v}_{t+1}-a_{t}}^2 \nonumber \\
        = & \norm{\frac{t-1}{t} \overline{v}_{t} + \frac{1}{t} v(\bfstrat^{(t)}) -a_{t}}^2 \nonumber \\
        = & \norm{\frac{t-1}{t} \left( \overline{v}_{t}   -a_{t} \right)  +  \frac{1}{t} \left(v(\bfstrat^{(t)})   -a_{t} \right)}^2 \nonumber \\
        = & \left(\frac{t-1}{t}\right)^2 \norm{ \overline{v}_{t}   -a_{t} }^2 \nonumber \\
        & + \left(\frac{1}{t}\right)^2\norm{  v(\bfstrat^{(t)})   -a_{t}}^2 \nonumber \\
        & + 2 \frac{t-1}{t^2} ( v(\bfstrat^{(t)} ) -a_{t}) \cdot (\overline{v}_{t}   -a_{t})
    \end{align*}
    
    Let $u_{\max{}} := \max \{ \norm{\utility_1}_\infty, \dots,  \norm{\utility_\numplayers}_\infty, \norm{\specutility}_\infty \}$ and $u_{\min{}} := -\max \{ \norm{-\utility_1}_\infty, \dots,  \norm{-\utility_\numplayers}_\infty, \norm{-\specutility}_\infty \}$.
    Recall that by construction of $G$, $u_{\max{}}, u_{\min{}}$ is polynomial bounded by $|I|$ so without loss of generality, let us fix utilities to a unit ball: $u_{\max{}} \leq 1, u_{\min{}} \geq -1$.
    We can rearrange the equation to obtain,
    \begin{align*}
        t^2 \norm{\overline{v}_{t+1}-a_{t+1}}^2 & - (t-1)^2 \norm{ \overline{v}_{t}   -a_{t} }^2 
        \nonumber \\
        & \leq (m+1) + 2(t-1)( v(\bfstrat^{(t)} ) -a_{t}) \cdot (\overline{v}_{t}   -a_{t}) \nonumber 
    \end{align*}
    Then sum both sides of the inequality for $t=1,\dots,n$ with the left-hand side telescoping to become $n^2 \norm{\overline{v}_{n+1}-a_{n+1}}^2$. Dividing both sides by $n^2$,
    \begin{align*}
     &\norm{\overline{v}_{n+1}-a_{n+1}}^2
   \nonumber 
   \\& \leq \frac{(m+1) }{n} + \frac{2}{n} \sum_{t=1}^n \frac{t-1}{n}( v(\bfstrat^{(t)} ) -a_{t}) \cdot (\overline{v}_{t}   -a_{t}) \nonumber \\
   & \leq \frac{(m+1) }{n} + \frac{2}{n} \sum_{t=1}^n \frac{t-1}{n} \beta_t ( v(\bfstrat^{(t)} ) -a_{t}) \cdot \norm{\overline{v}_{t}   -a_{t}} \nonumber \\
    \end{align*}
    Because $\frac{t-1}{n} \norm{\overline{v}_{t}   -a_{t}} \in [0,2]$, $\beta_t, -a_t \geq 0$ component-wise, and by construction of $\bfstrat_t$,
    \begin{align*}
     \norm{\overline{v}_{n+1}-a_{n+1}}^2
   \nonumber 
   & \leq \frac{(m+1) }{n} + \frac{8}{n} \sum_{t=1}^n \beta_t ( v(\bfstrat^{(t)} ) -a_{t}) \nonumber\\
   & \leq \frac{(m+1) }{n} + \frac{8}{n} \sum_{t=1}^n \beta_t ( v(\bfstrat^{(t)} )) \nonumber
   \\
   &\leq \frac{(m+1) }{{n}}
    \end{align*}
    By triangle inequality, after $n$ steps, we have a $\frac{2}{\sqrt{n}}$-CCE $\bfstrat$ with $\overline{\specutility}(\bfstrat) \geq y - \frac{2}{\sqrt{n}}$.
    Substituting $\epsilon = \frac{2}{\sqrt{n}}$, we have a Poly($\frac{1}{\epsilon}, n, m$) runtime concluding our proof of this lemma.
    
    Suppose that $\bfstrat^t$ is chosen where we only have the guarantee that $(\beta_t)_{1:m} v(\bfstrat^t)_{1:m} \leq \epsilon_1$ and  $(\beta_t)_{m+1} v(\bfstrat^t)_{m+1} \leq \epsilon_2$.
    Then, we would instead have $\norm{(\overline{v}_{n+1})_{1:m}-(a_{n+1})_{1:m}}^2 \leq \frac{(m+1) }{n} + 8 \epsilon_1$ and $\norm{(\overline{v}_{n+1})_{m+1}-(a_{n+1})_{m+1}}^2 \leq \frac{(m+1) }{n} + 8 \epsilon_2$.
\end{proof}

\subsection{Proof of Lemma \ref{lemma:halfspace}}

\halfspace*

\begin{proof}
We will prove that this oracle exists by construction.
We will further limit our choices of candidates $\bfstrat$ to product distributions.
Then, finding our desired oracle is equivalent to solving an optimization problem over $m$ variables: $x_1^1,\dots, x_1^{m_1}, \dots, x_\numplayers^1, \dots, x_\numplayers^{m_\numplayers}$ where $x_i^j$ corresponds to the probability that agent $i$ plays action $j$ under $\bfstrat$.
The objective is minimizing
\begin{align*}
    &\beta v(\bfstrat) = 
     - \beta_{m+1} (\overline{\specutility}(\bfstrat) - y) \nonumber \\
     &- \sum_{i=1}^\numplayers \left(\sum_{j=1}^{m_i} \beta_i^j\right) \left(\overline{\utility}_i(\bfstrat)  - \frac{\sum_{j=1}^{m_i}  \beta_i^j \sum_{\bfaction_{-i}} x_{-i}^{\bfaction{-i}}\utility_i(j, \bfaction_{-i})}{\sum_{j=1}^{m_i} \beta_i^j}  \right) \nonumber\\
     & =
      -\beta_{m+1} (\overline{\specutility}(\bfstrat) - y) \nonumber\\
      & - \sum_{i=1}^\numplayers \norm{\beta_i}_1 \left(\overline{\utility}_i(\bfstrat)  -\overline{\utility}_i\left(\frac{\beta_i}{\norm{\beta_i}_1}, \bfstrat_{-i}\right)  \right) 
\end{align*}
down to a non-positive value.
For convenience, we'll write,
\begin{align*}
    &C(\bfstrat) = 
     - \beta_{m+1} \overline{v}(\bfstrat) \nonumber \\
      & - \sum_{i=1}^\numplayers \norm{\beta_i}_1 \left(\overline{\utility}_i(\bfstrat)  -\overline{\utility}_i\left(\frac{\beta_i}{\norm{\beta_i}_1}, \bfstrat_{-i}\right)  \right) 
\end{align*}
Note that efficiently finding a $\bfstrat$ with a guaranteed negative upper bound lets us specify that upper bound as a condition $y$. 

This optimization problem is combinatorial and generally intractable.
However, we can exploit the usually high social welfare of no-regret learning.
Consider no-regret learning with
$n$ agents that have the cost function,
\begin{align*}
    c_i(\bfstrat) = - \frac{\beta_{m+1}\overline{\specutility}(\bfstrat)}{n} - \norm{\beta_i}_1 \left(\overline{\utility}_i(\bfstrat) -\overline{\utility}_i\left(\frac{\beta_i}{\norm{\beta_i}_1}, \bfstrat_{-i}\right)\right) ,
\end{align*}
with action space $\actionset_i$.
Since applying regret matching would ignore our $\overline{\specutility}$ terms, we instead take advantage of smoothness.
Observe that the first term (the game of penalizing principal agent) is bound by the smoothness of $\tilde{\game}$, and the second term (the game of minimizing incentive to deviate from $\beta$) is bound by the smoothness of $\game$.
Thus, for any $\bfstrat, \bfstrat^*$:
\begin{align*}
 \sum_{i=1}^N \frac{\overline{\specutility}(\bfstrat_{-i}, x^*_i)_{m+1}}{N}
    \leq \sum_{i=1}^N - \lambda_{\tilde{\game}} \frac{\overline{\specutility}(\bfstrat^*)}{N} - \mu_{\tilde{\game}} \frac{\overline{\specutility}(\bfstrat)}{N};
\end{align*}
\begin{align*}
  \sum_{i=1}^N \overline{\specutility}(\bfstrat_{-i}, x^*_i)_{1:m} 
  \leq 
     &\sum_{i=1}^N ||\beta_i|| \lambda_\game \left(\overline{\utility}_i\left(\beta\right) - \overline{\utility}_i(\beta) \right) \\
    & + ||\beta_i|| \mu_\game \left(\overline{\utility}_i\left(\frac{\beta_i}{\norm{\beta_i}_1}, \bfstrat_{-i}\right) - \overline{\utility}_i(\bfstrat) \right).
\end{align*}

We can now run a generic no-regret learning algorithm, say Hedge, which has regret $O(\sqrt{T})$ after $T$ episodes.
Specifically, the ergodic average $\bfstrat$ of $T$ episodes 
is a $O(\frac{1}{\sqrt{T}})$-CCE.
Letting $\bfstrat^*$ be a CCE for the game defined by $\{c_i\}_{i=1}^N$, it follows:
\begin{align*}
\beta v(\bfstrat) \leq&  -\beta_{m+1} (\lambda_{\tilde{\game}} \overline{\specutility}(\bfstrat^*) + \mu_{\tilde{\game}} \overline{\specutility}(\bfstrat)) \\
& + \mu_{\game} \beta_{1:m} v(\bfstrat)_{1:m} + O(\sqrt{T}).
\end{align*}

Finally, recall we require that $C(\bfstrat) + \beta_{m+1} y = \beta v(x) \leq 0$, meaning that in order for our halfspace oracle to always return something, we require $y$ satisfy:
\begin{align*}
    & \frac{\lambda \max_{\bfstrat \in \CCE}\overline{\specutility}(\bfstrat)}{\mu-1} + \beta_{m+1}y \leq 0 \\
     & \rightarrow y \leq \frac{\lambda \max_{\bfstrat \in \CCE}\overline{\specutility}(\bfstrat)}{\beta_{m+1}(1-\mu)} \leq \frac{\lambda \max_{\bfstrat \in \CCE}\overline{\specutility}(\bfstrat)}{(1-\mu)}
\end{align*}

\end{proof}

\subsection{Proof of Theorem \ref{thm:mainthm}}
\mainthm*

The main theorem follows directly from Lemmas \ref{lemma:blackwell}, \ref{lemma:halfspace}.
Specifically, for any $y \leq \max_{\bfstrat \in \MNE}$, instantiate the algorithm described in Lemma \ref{lemma:blackwell} with the oracle described in Lemma \ref{lemma:halfspace}.
This is described in full in Algorithm~\ref{alg:theoryinner}.
Note that the oracle returns a mixed strategy that is guaranteed to be composed of a polynomial number of known product distributions; hence, any mixed strategies returned by the oracle can be evaluated in polynomial time.

\subsection{Proof of Proposition \ref{prop:generalization}}
We first motivate our study of how behavioral deviations affect the generalization of strategic policies.
Specifically, we will look at examples of where naive strategic decision-making can result in significant regret when confronted with boundedly rational agents---even when said agents are asymptotically rational.

For simplicity, we'll discuss a traditional single-follower Stackelberg setting where $\numplayers = 1$.
In this setting, we'll refer to agent $0$ as the leader and to agent $1$ as the follower.
Recall that in fully rational settings, we can anticipate the follower to play a best response to whatever mixed strategy the leader adopts: $\behavior = \text{BR}(x_0) := \argmax_{a_1 \in \actionset_1} \EE_{a_0 \sim x_0} [\utility_1(a_0, a_1)]$.

Consider an instance of bounded rationality where the leader plays a strategy $x_0$ but the follower perceives the leader as having committed to the strategy $\hat{x}_0$.
This may arise in partial observability settings or when the follower is limited to bandit feedback, for example.
A classical statement of the fragility of Stackelberg equilibria is that even when the follower's error, $\norm{x_0 - \hat{x}_0}$, is infinitesimally small, the follower's bounded rationality can cost the leader a constant value bounded away from zero.
\begin{proposition}
        Fix a leader mixed strategy $x_0 \in \pdfover{ \actionset_0}$.
        There exists Stackelberg game ($\numplayers = 1$) with bounded payoffs where there is a sequence of mixed strategies $x_0^{(1)}, x_0^{(2)}, \dots \in \pdfover{\actionset_0}$ such that $\lim_{t \rightarrow \infty} \norm{x_0 - x_0^{(t)}}_\infty = 0$ but for all $t$, $\min_{x_1 \in \text{BR}(x_0)} \utility_0(x_0, x_1) - \max_{x_1 \in \text{BR}(x_0^{(t)})} \utility_0(x_0, x_1)$  is positive and bounded away from zero.
\end{proposition}

These situations arise when the relationship between the leader and follower's utility functions is not smooth.
Samuelson's game is an example of where this may result in non-robust optimal strategies independent of a Stackelberg assumption.
Consider the below payoff matrix.

\begin{tabular}{| c | c | c | }
\hline
& L & R \\
\hline
T & 100, 100 & 50, 99 \\
\hline
B & 99, 100 & 99, 99 \\
\hline
\end{tabular}

There is a unique Nash equilibria at $a_0 = T, a_1=L$.
Accordingly, optimizing against the best response of the column-player would recommend a $T$ pure strategy to the row-player.
However, if the column-player is boundedly rational or otherwise acts according to a perturbed payoff matrix, a safer strategy for the row-player is to play the action $B$, which guarantees at least a payoff of $99$.
Optimizing against the worst-case $\epsilon$-best response of the column-player, where $\epsilon \geq 1$, would recommend this robust $B$ strategy to the row-player.
Extending this insight into the $\numplayers > 1$ domain, we similarly want to optimize against worst-case equilibria to afford robustness against payoff structures such as Samuelson games.

\generalization*
The proof of this proposition follows immediately from the observation that all three uncertainties---subjective rationality, procedural rationality, and myopicness/commitment power---can be framed as uncertainty about agent utility functions.

\def\rewuncertainty{\Xi}

\begin{proof}
First, subjective rationality is by definition an uncertainty about agent utility functions.
Our parameterization of this uncertainty yields a natural uncertainty set over reward functions.
Recall that we upper bound the infinity norm by $\gamma_s$, which is defined as the $\inf_{\gamma} \text{ s.t. } \norm{u_i - \hat{u}_i}_\infty := \max_{\bfaction \in \actionset} |u_i(\bfaction) - \hat{u}_i(\bfaction) | \leq \gamma$.
We accordingly define the uncertainty set $\rewuncertainty_\varepsilon(u_i){} = \{u_i: \actionset \rightarrow \reals \mid \norm{u_i - \hat{u}_i}_\infty \leq \varepsilon\}$.

Similarly, we can define an uncertainty set around non-myopic agents and general uncertainty about our commitment power by upper-bounding possible perturbations to an agent's originally anticipated utility function.
In particular, $| u_i  - \left(u_i(\bfstrat) + \sum_{t=1}^\infty \gamma_m^t u_i(\bfstrat_t) \right)| \leq \frac{\norm{u_i}_\infty }{1 - \gamma_m} = \epsilon$.
Hence, all non-myopic agents with $\gamma_m \leq 1 - \frac{\norm{u_i}_\infty }{\epsilon}$ will be expressed in the uncertainty set $\rewuncertainty_\varepsilon(u_i){}$.

Procedural rationality can also be framed in terms of $\rewuncertainty_\varepsilon(u_i){}$.
In particular, the $\epsilon_i$-coarse correlated equilibria under $u_i$ exactly coincides with the union of coarse correlated equilibria under reward perturbations of up to $\epsilon/2$.
Formally, $\bfstrat$ is in an $\epsilon_i$-coarse correlated equilibria under $u_i$ if and only if,
\begin{align*}
   \forall i \in [1, \dots, \numplayers], x_i \in \pdfover{ \actionset_i}: \overline{u}_i(\bfstrat) + \epsilon_i \geq \overline{u}_i(x_i, \bfstrat_{-i}) \nonumber
\end{align*}
$\bfstrat$ is in a coarse correlated equilibria under some ${u'}_i \in \rewuncertainty_\varepsilon(u_i)$ if and only if,
\begin{align*}
    \exists& {u'}_i \in \rewuncertainty_\varepsilon(u_i), \text{ s.t. } \\
    &   \forall i \in [1, \dots, \numplayers], x_i \in \pdfover{ \actionset_i}: \overline{u'}_i(\bfstrat) \geq \overline{u'}_i(x_i, \bfstrat_{-i}), \nonumber
\end{align*}
which is equivalent to the condition,
\begin{align*}
    &   \forall i \in [1, \dots, \numplayers], x_i \in \pdfover{\actionset_i}: \\
    & \overline{u'}_i(\bfstrat) + \varepsilon \geq \overline{u'}_i(x_i, \bfstrat_{-i}) - \varepsilon, \nonumber
\end{align*}

Hence, the set of $\frac{\varepsilon}{2}$-equilibria behaviors, up to $\left(1 - \frac{\norm{u_i}_\infty }{\varepsilon}\right)$-discount factor non-myopic agent behaviors, and up to $\varepsilon$ reward perturbation-consistent behaviors are all contained in the set of $\varepsilon$-CCE.
\end{proof}

\section{Appendix: Additional Experiments}
\subsection{Ablation Study: Dynamic vs Fixed Lagrange multipliers.}
In our robust learning framework, Algorithm \ref{alg:blackboxinner}, the Lagrange multipliers $\bm{\lambda}$ play an important rule in moderating the self-play dynamics used to sample adversarial dynamics.
Specifically, the multiplier $\lambda_i$ for agent $i$ balances agent $i$'s incentive to improve its own reward with its obligation to act adversarially to the \robustagent{} by minimizing the \robustagent{}'s reward.
Recall that smaller values of $\lambda$ yield more antagonistic agent objectives.
As described in Algorithm \ref{alg:dynamicmechanism}, these multipliers $\bm{\lambda}$ are periodically updated using local Monte-Carlo estimates of regret.
This raises the question of whether we can instead fix a constant non-zero value for the multipliers $\bm{\lambda}$ and still retain an effective adversarial equilibria sampler.
This experiment answers that question in the negative.
Figure \ref{fig:ablation} visualizes the equilibria discovered with a fixed $\bm{\fixedlambda}$ in the same format as in Figure \ref{fig:envs:bimatrix} (middle).
This comparison shows that using a fixed $\bm{\fixedlambda}$ affects the equilibria discovered by \ouralgo{}; the bottom right quadrant which contains the $\epsilon$-equilibria discovered with the dynamic multipliers $\bm{\lambda}$ used by Algorithm \ref{alg:dynamicmechanism} are not reached for any values of frozen multipliers $\bm{\fixedlambda}$.
This demonstrates that, even in simple game settings, certain $\epsilon$-equilibria are only reachable with dynamic $\bm{\lambda}$ and hence multiplier updates are necessary for proper behavior.

\begin{figure}[ht]
    \centering
    \includegraphics[width=0.8\linewidth]{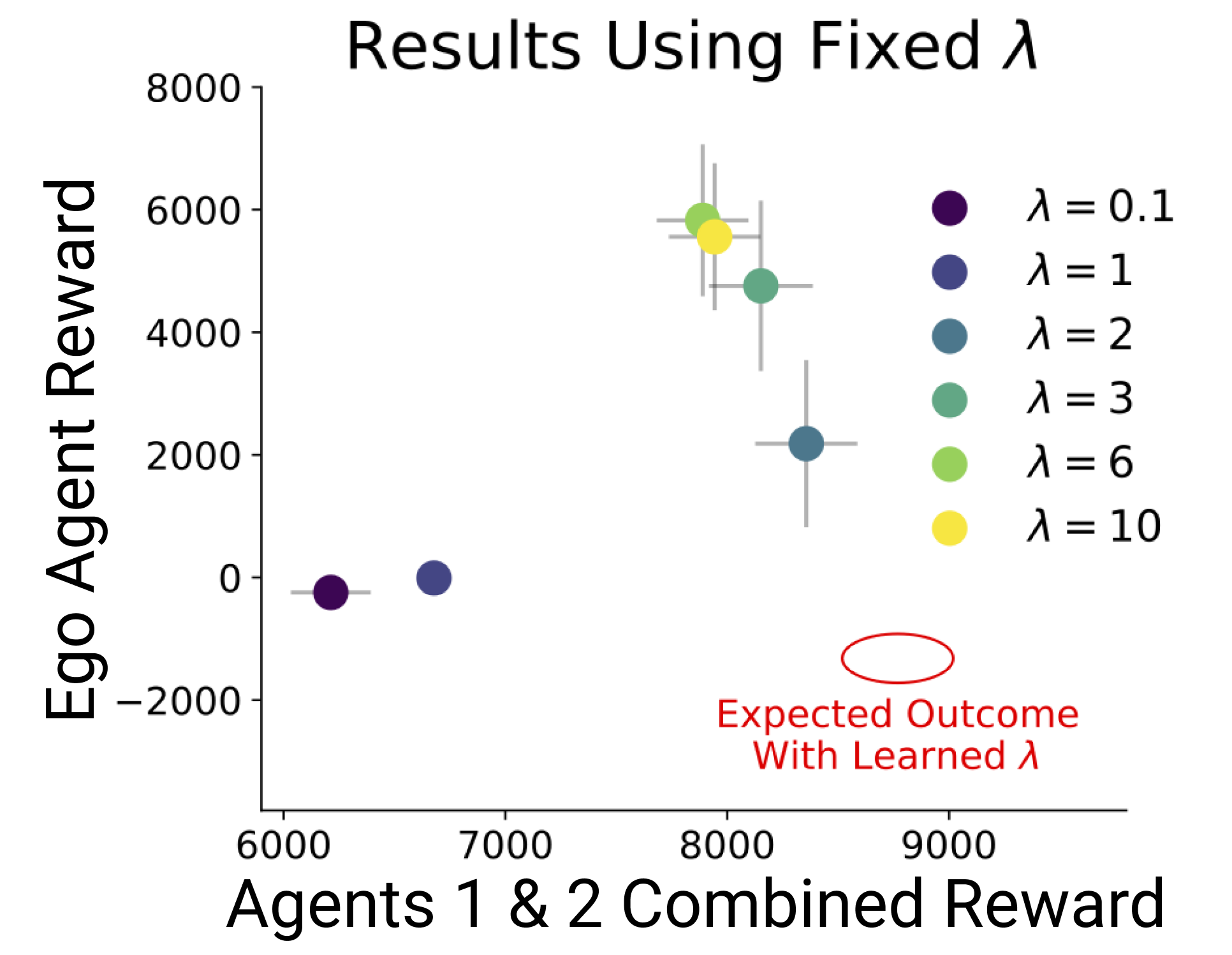}
    \caption{
    Using fixed values of $\lambda$ (rather than allowing it to update, as in the full algorithm) distorts performance and prevents agents from reaching the same $\epsilon$-equilibria discovered with learned $\lambda$.
    }
    \label{fig:ablation}
\end{figure}

\label{sec:spatiotemporalgame}
\begin{figure}[ht]
    \centering
    \includegraphics[width=0.8\linewidth]{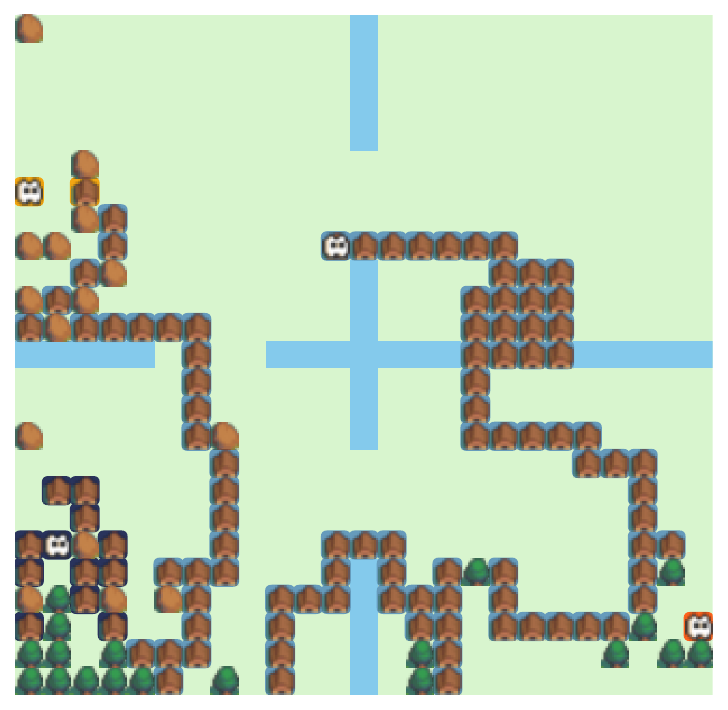}
    
    \caption{
    Visualization of the spatiotemporal economic simulation published by \cite{zheng_ai_2020} and used in our dynamic mechanism design experiment in Figure \ref{fig:aieperf}. In this 25-by-25 grid world, 4 heterogenous agents perform labor, trade resources, earn income, and pay taxes according to the schedules set by our policies.
    }
    \label{fig:spatiotemporalgame}
\end{figure}
\subsection{Additional Experiment Details}
Experiment source code is available at \url{https://github.com/salesforce/strategically-robust-ai}.

The hyperparameters for our deep dynamic mechanism design experiments are listed in Table \ref{tab:nmatrixparams} for the N-Matrix games (depicted in Figure \ref{fig:nmatrixperf}), and Table \ref{tab:aieparams} (depicted in Figure \ref{fig:aieperf}) for the AI Economist tax policy game.
Additional hyperparameters specific to the AI Economist simulation environment were kept at default values as described in the manuscript \cite{zheng_ai_2020}.

\textbf{Table \ref{tab:nmatrixparams} details}:
In the N-Matrix game experiments, we train 20 models for each training method and conduct evaluation runs of each model in every test environment.
However, evaluation runs may fail to stabilize.
Thus, we drop the bottom 10 evaluation runs for each training method.

\textbf{Table \ref{tab:aieparams} details}:
In the AI Economist tax policy experiments, we train 9 models for each training method and conduct evaluation runs of each model in every test environment.
However, evaluation runs may fail to stabilize.
Thus, we drop the bottom 6 evaluation runs for each training method.
Furthermore, for this complex task, the adversary in Algorithm \ref{alg:dynamicmechanism} may fail to converge in reasonable time.
Thus, for the Algorithm \ref{alg:dynamicmechanism} training method, we select its 9 trained models from a set of 15 candidate models (i.e., dropping at most 6 failed training runs).
This selection is made programmatically on the basis of each model's median reward on validation runs in a standalone validation environment ($\eta=0.27$).
We similarly choose the $\epsilon=30$ hyperparameter for Algorithm~\ref{alg:dynamicmechanism} from a grid-search over $\epsilon \in \{0, -10, -20, -30, -40\}$ on the basis of validation reward.

\section{Appendix: Additional Related Work}
\label{app:relatedwork}
In this section, we discuss additional related work regarding adversarially robust reinforcement learning.
The topic of adversarially robust reinforcement learning originates from control theory literature.
\citet{morimoto_robust_2001} proposed a form of adversarially robust reinforcement learning, inspired by $\gH_\infty$ control, for attaining robustness to uncertainty about the environment dynamics.
Later, \citet{pinto_robust_2017} proposes to learn an adversary policy to induce this robustness, again targeting robustness to uncertainty about environment dynamics in a control-theoretic sense.
Other works have since studied perturbing transition matrices, observation spaces, and action probabilities of Markov decision processes, with motivating applications in robotics and control theory \citep{tessler_action_2019, hou_robust_2020}.
These concepts have recently been extended to multi-agent settings, such as by \citep{li_robust_2019}.
However, these prior work on robust multi-agent learning differ from ours in two important ways.
First, their multi-agent task is controlling a swarm (of multiple agents) in an efficient and coordinated fashion.
Our multi-agent task is designing a policy that will yield favorable strategic outcomes in a multi-agent game.
Second, their works analyze control-theoretic notions of robustness, extending single-agent policy techniques to a multi-agent policy setting.

\begin{table}[H]
\caption{Hyperparameters for N-Matrix Experiments}
\label{tab:nmatrixparams}
\begin{tabular}{ll}
Parameter                                       & Value       \\ \hline
Training algorithm                              & PPO         \\
Episodes                                        & 10,000       \\
CPUs (Baselines)                                & 15          \\
CPUs (Algorithm \ref{alg:dynamicmechanism})     & 95          \\
                                                &             \\
Environment parameters                          &             \\ \hline
Number of agents                                & 4           \\
Episode length                                  & 500         \\
World dimensions                                & 7x7x7x7     \\
Training Seeds                                  & 20          \\
Test Seeds                                      & 10          \\
                                                &             \\
Neural network parameters                       &             \\ \hline
Number of convolutional layers                  & 0           \\
Number of fully-connected layers                & 2           \\
Fully-connected layer dimension (agent)         & 128; 32     \\
Fully-connected layer dimension (planner)       & 128; 32     \\
LSTM cell size (agent)                          & 0           \\
LSTM cell size (planner)                        & 0           \\
All agents share weights                        & True        \\
                                                &             \\
PPO Parameters                                  &             \\ \hline
Learning Rate (Principal)                       & 0.0006      \\
Learning Rate, $\eta$ (Other Agents)            & 0.003       \\
Entropy regularization (Principal)              & 0.1         \\
Entropy regularization, $\alpha$ (Other Agents) & 0.025       \\
Gamma                                           & 0.998       \\
GAE Lambda                                      & 0.98        \\
Gradient clipping                               & 10          \\
Value function loss coefficient                 & 0.05        \\
SGD Minibatch Size                              & 2500        \\
SGD Sequence Length                             & 50          \\
Value/Policy networks share weights             & False       \\
                                                &             \\
Algorithm \ref{alg:dynamicmechanism} Parameters &             \\ \hline
$n_\text{train}$                                & 4           \\
$n_\text{test}$                                 & 10          \\
Initial multipliers $\lambda$                   & 8           \\
Multiplier learning rate $\alpha_\lambda$       & 0.01, 0.001
\end{tabular}
\end{table}

\begin{table}[H]
\caption{Hyperparameters for AI Economist Experiments}
\label{tab:aieparams}
\begin{tabular}{ll}
Parameter                                       & Value   \\ \hline
Training algorithm                              & PPO     \\
Pretrained Episodes                             & 450,000 \\
Finetuned Episodes                              & 20,000  \\
CPUs (Baselines)                                & 15      \\
CPUs (Algorithm \ref{alg:dynamicmechanism})     & 95      \\
                                                &         \\
Environment parameters                          &         \\ \hline
Number of agents                                & 4       \\
Episode length                                  & 1000   \\
World dimensions                                & 25x25   \\
Default iso-elastic $\eta$                      & 0.23    \\
Training Seeds                                  & 9       \\
Test Seeds                                      &  3       \\
                                                &         \\
Neural network parameters                       &         \\ \hline
Number of convolutional layers                  & 2       \\
Number of fully-connected layers                & 2       \\
Fully-connected layer dimension (agent)         & 128     \\
Fully-connected layer dimension (planner)       & 256     \\
LSTM cell size (agent)                          & 128     \\
LSTM cell size (planner)                        & 256     \\
All agents share weights                        & True    \\
                                                &         \\
PPO Parameters                                  &         \\ \hline
Learning Rate (Principal)                       & 0.0001  \\
Learning Rate, $\eta$ (Other Agents)            & 0.0003  \\
Entropy regularization (Principal)              & 0.1     \\
Entropy regularization, $\alpha$ (Other Agents) & 0.025   \\
Gamma                                           & 0.998   \\
GAE Lambda                                      & 0.98    \\
Gradient clipping                               & 10      \\
Value function loss coefficient                 & 0.05    \\
SGD Minibatch Size                              & 3000    \\
SGD Sequence Length                             & 50      \\
Value/Policy networks share weights             & False   \\
                                                &         \\
Algorithm \ref{alg:dynamicmechanism} Parameters &         \\ \hline
$n_\text{train}$                                & 2       \\
$n_\text{test}$                                 & 5       \\
Initial multipliers $\lambda$                   & 20      \\
Multiplier learning rate $\alpha_\lambda$       & 0.01   
\end{tabular}
\end{table}

\newpage
\section{Appendix: Algorithm Pseudocodes}
In this section, we provide the full pseudocode description of Algorithms \ref{alg:theoryinner} and \ref{alg:dynamicmechanism}.
\begin{algorithm}[!h]
\caption{
    Poly($n, m, \rho, \varepsilon$)-time algorithm for sampling approximately optimal CCE.
}
\label{alg:theoryinner}
\begin{algorithmic}
\STATE \textbf{Output:} $\epsilon$-CCE $\bfstrat$ such that
\begin{align*}
    \overline{\specutility}(\policy_0, \bfstrat) \geq \frac{1}{\rho} \max_{\tilde{\bfstrat} \in \epsilon\mhyphen \CCE} \overline{\specutility}(\policy_0, \tilde{\bfstrat}) - \varepsilon.
\end{align*}
\STATE \textbf{Input:} Generalization parameter $\epsilon \in \reals_+^m$, tolerance $\varepsilon$, objective function $\specutility: \actionset_0 \times \actionset \rightarrow \reals$, \robustagent{} \ourpolicy{} $\policy_0$, bandit feedback access to utility functions $\overline{\utility}_0, \dots, \overline{\utility}_\numplayers$.

\WHILE{Binary search over $y \in \reals$ (within $\specutility$'s payoff bounds) is coarser than the desired Poly($|I|$) resolution}
\STATE Initialize some $\bfstrat \in \prodpdfover{\actionset}$.
\FOR{$O(\frac{1}{\varepsilon^2})$ iterations indexed by $t$}
    \STATE Compute the vector $v_t$ from Theorem \ref{thm:mainthm} (Eq \ref{eq:v}).
    \STATE Compute importance weight vector,
    \begin{align*}
        \beta_t \leftarrow
        \begin{bmatrix}
            \max\{v_1 - \epsilon_1,0\} \\
            \vdots \\
            \max\{v_{m_1} - \epsilon_1,0\} \\
            \vdots\\
            \max\{v_{m-1} - \epsilon_\numplayers,0\} \\
            \vdots \\
            \max\{v_{m} - \epsilon_\numplayers,0\} \\
            \max\{v_{m+1},0\} \\
        \end{bmatrix}
    \end{align*}
    \STATE Normalize $\beta_t \leftarrow \beta_t / \norm{\beta_t}_1$.
    \STATE Run $O(\frac{1}{\varepsilon^2})$ steps of simultaneous deterministic no-regret dynamics on $n$ agents with cost function,
    \begin{align*}
    &c_i(\bfstrat) = - \frac{\beta_{m+1}\overline{\specutility}(\bfstrat)}{n} \\
    &- \norm{\mathbf{\beta}_i}_1 \left(\overline{\utility}_i(\bfstrat) -\overline{\utility}_i\left(\frac{\mathbf{\beta}_i}{\norm{\mathbf{\beta}_i}_1}, \bfstrat_{-i}\right)\right) ,
    \end{align*}
    where $\mathbf{\beta}_i$ are the components of $\beta$ starting with index $1 + \sum_{j=1}^{i-1} m_j$ and ending with index $\sum_{j=1}^{i} m_j$ inclusive.
    \STATE Store empirical play distribution of the no-regret dynamics as $\bfstrat_t$.
    \STATE If $\beta_t \bfstrat_t \geq \varepsilon$, $y$ is too large. Terminate loop and resume binary search.
\ENDFOR
\STATE If the prior loop was successful ($y$ was not too large), store the average empirical play distribution over $\bfstrat_1, \bfstrat_2, \dots$ as $\bfstrat^*$.
\ENDWHILE
\STATE Return $\bfstrat^*$.
\end{algorithmic}
\end{algorithm}

Recall that the primary difference between Algorithm \ref{alg:blackboxinner} and Algorithm \ref{alg:theoryinner} is that Algorithm \ref{alg:theoryinner} invokes no-regret learning algorithms as a subprocedure, which are provably computationally inefficient (e.g., see \citet*{hazan_computational_2016}). No-regret algorithms do not even exist for the online learning of some finite VC dimension classes \cite{littlestone_learning_1987}, meaning there are no known implementations of Algorithm \ref{alg:theoryinner} for the experiments considered in the paper.
In situations where one can reasonably expect gradient descent self-play to behave like a no-regret algorithm, Algorithm \ref{alg:blackboxinner} can still be seen as an implementation of Algorithm \ref{alg:theoryinner}.

\begin{algorithm}[ht]
\caption{Adversarially robust dynamic mechanism design.}
\label{alg:dynamicmechanism}
\begin{algorithmic}
\STATE \textbf{Output:} Learned parameters, $\theta_0$, for a mechanism represented as agent $0$.
\STATE \textbf{Input:}  Reward slack $\epsilon$, learning rate $\alpha_{\lambda}$, batch size $\beta$, initial agent parameters $\bm{\theta} = [\theta_1, \dots, \theta_\numplayers]$.
\STATE \textbf{Input:} Number of rounds $n_{\text{rounds}}$, evaluation batches $n_{\text{eval}}$, training batches $n_{\text{train}}$.
\FOR{$j = 1, \dots, n_{\text{rounds}}$}
    \STATE Copy $\bm{\theta}$ into placeholders $\bm{\tilde{\theta}} \leftarrow \bm{\theta}$.
    \FOR{$j = 1, \dots, n_{\text{eval}}$}
        \STATE Accumulate $\beta$ timesteps under $\bm{\tilde{\theta}}$, $\theta_0$.
        \STATE For each agent $i \geq 1$, store their experience in a tuple of lists,
        \begin{align*}
            B_{i,j} := (\text{Rew}_{i,j}, \text{Obs}_{i,j}, \text{Actions}_{i,j}). \nonumber
        \end{align*}
        \STATE Update each agent $i \geq 1$ critic (if used) and actor for $\bm{\tilde{\theta}}$ with $B_i$.
    \ENDFOR
    \STATE Update each agent $i \geq 1$ multiplier:
    \begin{align*}
        \lambda_i \leftarrow \lambda_i - \alpha_{\lambda} \text{Mean} \brck{\text{Rew}_{i,n_{\text{rounds}}} - \epsilon - \text{Rew}_{i,0}} \nonumber
    \end{align*}
    \FOR{$j = 1, \dots, n_{\text{train}}$}
        \STATE Accumulate $\beta$ timesteps under $\bm{{\theta}}$ and $\theta_0$ and store them in $B_{0,j},\dots, B_{\numplayers,j}$.
        \STATE Update the mechanism's $\theta_0$ critic and actor with $B_0$ using a slow learning rate.
        \STATE Update each agent $i \geq 1$ critic for $\bm{\theta}$ (if used) with $B_i$.
        \STATE Update each agent $i \geq$ actor for $\bm{\theta}$with,
        \begin{align*}
            \tilde{B}_{i,j} = \left(\tilde{\text{Rew}}_{i,j}, \text{Obs}_{i,j}, \text{Actions}_{i,j}\right) \nonumber
        \end{align*}
        where the recorded reward of agent $i$, at each timestep $t=1,\dots,\beta$, is modified:
        \begin{align*}
            \tilde{\text{Rew}}_{i,j} = \left[\frac{\lambda_i \text{Rew}_{i,j}^{(t)} - \text{Rew}_{0,j}^{(t)}}{1 + \lambda_i} \mid t=1,\dots, \beta\right] \nonumber
        \end{align*}
    \ENDFOR
\ENDFOR
\STATE Return parameters $\theta_0$.
\end{algorithmic}
\end{algorithm}

\end{document}